\documentclass{article}

\usepackage{microtype}
\usepackage{graphicx}
\usepackage{subfigure}
\usepackage{booktabs} %

\usepackage[accepted]{icml2020}

\icmltitlerunning{Generalization Error of GLMs in High Dimensions}
\usepackage{algorithm,algorithmic}
\usepackage{verbatim}
\usepackage{amsmath}  
\usepackage{amssymb,amsthm,amsmath,amsfonts,mathtools}
\usepackage[inline]{enumitem}

\usepackage{mathtools}
\usepackage{tikz}
\usetikzlibrary{arrows}
\usepackage{amsthm}
\usepackage{bbm}

\def\ts{{\rm ts}}
\def\tr{{\rm tr}}

\renewcommand{\hat}{\widehat}

\def\beq{\begin{equation}}
\def\eeq{\end{equation}}
\def\beqa{\begin{eqnarray}}
\def\eeqa{\end{eqnarray}}
\def\beqan{\begin{eqnarray*}}
\def\eeqan{\end{eqnarray*}}

\def\half{{\tfrac{1}{2}}}

\def\R{{\mathbb{R}}}

\DeclareMathOperator*{\argmin}{argmin}

\DeclareMathOperator{\diag}{Diag}
\DeclareMathOperator{\prox}{prox} %

\def\mp{{\rm{mp}}}

\newcommand\independent{\protect\mathpalette{\protect\independenT}{\perp}}
\def\independenT#1#2{\mathrel{\rlap{$#1#2$}\mkern2mu{#1#2}}}
\def\inv{^{-1}}

\newtheorem{proposition}{Proposition}

\newtheorem{theorem}{Theorem}
\newtheorem{lemma}{Lemma}
\newtheorem{corollary}{Corollary}
\theoremstyle{definition}

\newtheorem{definition}{Definition}
\newtheorem{assumption}{Assumption}

\def\phat{\wh{p}}

\def\zhat{\widehat{z}}
\def\what{\widehat{w}}

\def\arr{\rightarrow}
\def\Exp{\mathbb{E}}
\def\Prob{\mathbb{P}}

\def\Cov{\mathrm{Cov}}

\def\alphabar{\overline{\alpha}}

\def\gammabar{\overline{\gamma}}

\def\PL2{\stackrel{PL(2)}{=}}

\def\kp{k\! + \!}
\def\lp{\ell\! + \!}
\def\lm{\ell\! - \!}
\def\Lm{L\! - \!}

\newcommand{\zero}{\mathbf{0}}

\newcommand{\dbf}{\mathbf{d}}

\newcommand{\fbf}{\mathbf{f}}
\newcommand{\gbf}{\mathbf{g}}

\newcommand{\pbf}{\mathbf{p}}
\newcommand{\pbfhat}{\widehat{\mathbf{p}}}
\newcommand{\qbf}{\mathbf{q}}

\newcommand{\rbf}{\mathbf{r}}

\newcommand{\sbf}{\mathbf{s}}

\newcommand{\ubf}{\mathbf{u}}

\newcommand{\wbfhat}{\wh{\wbf}}
\newcommand{\xbf}{\bm{x}}

\newcommand{\zbf}{\mathbf{z}}

\newcommand{\zbfhat}{\widehat{\mathbf{z}}}

\newcommand{\Hbf}{\mathbf{H}}

\newcommand{\Ibf}{\mathbf{I}}

\newcommand{\Kbf}{\mathbf{K}}

\newcommand{\Pbf}{\mathbf{P}}
\newcommand{\Phat}{\wh{P}}

\newcommand{\Sbf}{\mathbf{S}}
\newcommand{\Ubf}{\mathbf{U}}

\newcommand{\Vbf}{\mathbf{V}}

\newcommand{\What}{\wh{W}}

\newcommand{\Xbf}{\mathbf{X}}

\newcommand{\Zhat}{\wh{Z}}

\newcommand{\inner}[1]{\langle{#1}\rangle}

\def\Sigmabf{{\boldsymbol \Sigma}}

\def\xibf{{\boldsymbol \xi}}

\newcommand{\thetabar}{{\overline{\theta}}}

\newcommand{\phibf}{{\bm{\phi}}}

\newcommand{\indic}[1]{\mathbbm{1}_{ \{ {#1} \} }}

\newcommand{\tran}{^{\text{\sf T}}}

\def\PLeq{\stackrel{PL(2)}{=}}
\def\Norm{{\mathcal N}}

\def\alphabar{\overline{\alpha}}

\newcommand{\bkt}[1]{{\left< #1 \right>}}
\RequirePackage{xcolor,bm,setspace}
\RequirePackage{mathrsfs}

\providecommand{\black}[1]{\textcolor{black}{#1}}

\providecommand{\old}[1]{ }

\providecommand{\Var}{{\rm Var}}
\providecommand{\mc}{\mathcal}
\providecommand{\ie}{{\rm i.e.}}

\providecommand{\wb}{\overline}

\providecommand{\wh}{\widehat}
\newcommand{\norm}[1]{\left\|#1\right\|}
\providecommand{\Real}{\mathbb{R}}

\providecommand{\i}{\bm{i}}

\providecommand{\p}{\mathbf{p}}

\providecommand{\y}{\bm{y}}

\providecommand{\dbf}{\mathbf{d}}

\providecommand{\fbf}{\mathbf{f}}
\providecommand{\gbf}{\mathbf{g}}

\providecommand{\pbf}{\mathbf{p}}
\providecommand{\qbf}{\mathbf{q}}
\providecommand{\rbf}{\mathbf{r}}
\providecommand{\sbf}{\mathbf{s}}

\providecommand{\ubf}{\mathbf{u}}

\providecommand{\wbf}{\bm{w}}
\providecommand{\xbf}{\bm{x}}
\providecommand{\ybf}{\bm{y}}
\providecommand{\zbf}{\mathbf{z}}

\providecommand{\Hbf}{\mathbf{H}}
\providecommand{\Ibf}{\mathbf{I}}

\providecommand{\Kbf}{\mathbf{K}}

\providecommand{\Mbf}{\mathbf{M}}

\providecommand{\Pbf}{\mathbf{P}}

\providecommand{\Sbf}{\mathbf{S}}

\providecommand{\Ubf}{\mathbf{U}}
\providecommand{\Vbf}{\mathbf{V}}

\providecommand{\Xbf}{\mathbf{X}}

\providecommand{\phibf}{\mbf{\phi}}

\providecommand{\mcN}{\mathcal{N}}

\providecommand{\indic}[1]{\mathbbm{1}_{ \{ {#1} \} }}
\providecommand{\ignore}[1]{}

\usepackage{etoolbox}
\newtoggle{conference}
\togglefalse{conference} %

\usepackage[colorlinks,linkcolor=blue,citecolor=blue,bookmarks=false,pagebackref]{hyperref}

\begin{document}

\twocolumn[
\icmltitle{Generalization Error of Generalized Linear Models in High Dimensions}

\icmlsetsymbol{equal}{*}

\begin{icmlauthorlist}
\icmlauthor{Melikasadat Emami}{ucla}
\icmlauthor{Mojtaba Sahraee-Ardakan}{ucla}
\icmlauthor{Parthe Pandit}{ucla}
\icmlauthor{Sundeep Rangan}{nyu}
\icmlauthor{Alyson K. Fletcher}{ucla,uclastat}
\end{icmlauthorlist}

\icmlaffiliation{ucla}{Department of Electrical and Computer Engineering, University of California, Los Angeles, Los Angeles, USA}
\icmlaffiliation{uclastat}{Department of Statistics, University of California, Los Angeles, Los Angeles, USA}

\icmlaffiliation{nyu}{Department of Electrical and Computer Engineering, New York University, Brooklyn, New York, USA}

\icmlcorrespondingauthor{Melikasadat Emami}{emami@ucla.edu}
\icmlcorrespondingauthor{Mojtaba Sahraee-Ardakan}{msahraee@ucla.edu}
\icmlcorrespondingauthor{Parthe Pandit}{parthepandit@ucla.edu}
\icmlcorrespondingauthor{Sundeep Rangan}{srangan@nyu.edu}
\icmlcorrespondingauthor{Alyson K. Fletcher}{akfletcher@ucla.edu}

\icmlkeywords{Machine Learning, Generalization error, Double Descent, }
\vskip 0.3in
]
\printAffiliationsAndNotice{}

\begin{abstract}
At the heart of machine learning lies the question of generalizability of learned rules over previously unseen data.  While over-parameterized models based on neural networks are now ubiquitous in machine learning applications, our understanding of their generalization capabilities is incomplete.  This task is made harder by the non-convexity of the underlying learning problems.  
We provide a general framework to characterize the asymptotic generalization error for single-layer neural networks (i.e., generalized linear models) with arbitrary non-linearities,
making it applicable to regression as well as classification problems.
This framework enables analyzing the effect of (i) over-parameterization and non-linearity during modeling; and (ii) choices of loss function, initialization, and regularizer during learning.
Our model also captures mismatch between training and test distributions.
As examples, we analyze a few special cases, namely linear regression and logistic regression.  We are also able to rigorously and analytically explain the \emph{double descent} phenomenon in generalized linear models.  
\end{abstract}

\section{Introduction} \label{sec:intro}

A fundamental goal of machine learning is 
\emph{generalization}:   the ability to  draw 
inferences about unseen data from finite training examples.  
Methods to quantify
the generalization error are therefore critical in 
assessing the performance of any machine learning approach.

This paper seeks to characterize 
the generalization error for a class of
generalized linear models (GLMs) of the form
\begin{equation} \label{eq:glm}
    y = \phi_{\rm out}(\bkt{\xbf, \wbf^0}, d),
\end{equation} 
where $\xbf \in \R^p$ is a vector of input features, 
$y$ is a scalar output, 
$\wbf^0 \in \R^p$ are weights to be learned, $\phi_{\rm out}(\cdot)$
is a known link function, and $d$ is random noise.
The notation $\bkt{\xbf,\wbf^0}$ denotes an inner product.
We use the superscript ``$0$" to denote the ``true" values
in contrast to estimated or postulated quantities.
The output may be continuous or discrete to model either
regression or classification problems.  

We measure the generalization error in a standard manner:
we are given training
data $(\xbf_i,y_i)$, $i=1,\ldots,N$ from which we learn some parameter
estimate $\wbfhat$ via a regularized empirical risk minimization of the form
\beq  \label{eq:whatmin}
    \wbfhat = \argmin_{\wbf} 
      F_{\rm out}(\ybf,\Xbf\wbf) + F_{\rm in}(\wbf),
\eeq
where $\Xbf=[\xbf_1\,\xbf_2\,\ldots\,\xbf_N]\tran$, is the data matrix, $F_{\rm out}$ is some output 
loss function,
and $F_{\rm in}$ is some regularizer on the weights.
We are then given a new test sample, $\xbf_{\rm ts}$,
for which the true and predicted values
are given by
\beq \label{eq:ytest}
    y_{\rm ts} = \phi_{\rm out}(\bkt{\xbf_{\rm ts},\wbf^0},d_{\rm ts}), \quad
    \wh{y}_{\rm ts} = \phi(\bkt{\xbf_{\rm ts},\wbfhat}),
\eeq
where $d_{\rm ts}$ is the noise in the test sample, and
$\phi(\cdot)$ is a postulated inverse link function that may be
different from the true function $\phi_{\rm out}(\cdot)$.
The generalization error is then defined as the expectation of some expected loss between $y_{\rm ts}$ and $\wh{y}_{\rm ts}$
of the form
\beq \label{eq:gen_nolim}
    \Exp\,f_\ts(y_\ts,\wh y_\ts),
\eeq
for some test loss function $f_\ts(\cdot)$
such as squared error or prediction error.

Even for this relatively simple GLM model,
the behavior of the
generalization error is not fully understood.
Recent works \cite{montanari2019generalization,deng2019model,mei2019generalization,salehi2019impact} have 
characterized the
generalization error of various linear 
models for classification and regression 
in certain large random problem instances.
Specifically, the 
number of samples $N$ and
number of features $p$ both grow without bound
with their ratio satisfying
${p}/{N} \rightarrow \beta \in (0,\infty)$,
and the samples in the training data $\xbf_i$ 
are drawn randomly.  In this limit, the
generalization error can be exactly computed.
The analysis can explain the so-called 
\emph{double descent} phenomena \cite{belkin2019reconciling}:  
in highly under-regularized settings, the test error may
initially \emph{increase} with the number of data samples $N$
before decreasing.  
See the prior work section below for more details.

\paragraph*{Summary of Contributions.}
Our main result (Theorem~\ref{thm:thm1_complete}) 
provides a procedure for exactly computing
the asymptotic value of the
generalization error \eqref{eq:gen_nolim} for
GLM models
in a certain random high-dimensional regime called the Large
System Limit (LSL).
The procedure enables the generalization error to be
related to 
key problem parameters including the sampling ratio
$\beta=p/N$,
the regularizer, the output function, 
and the distributions of the true weights and noise.
Importantly, our result holds under very general 
settings including:
\begin{enumerate*}[label=(\roman*)]
\item arbitrary test metrics $f_\ts$; 
\item arbitrary training loss functions $F_{\rm out}$ as well as decomposable regularizers $F_{\rm in}$;
\item arbitrary link functions $\phi_{\rm out}$;
\item correlated covariates $\xbf$;
\item underparameterized ($\beta<1$) and overparameterized regimes ($\beta>1$); and
\item distributional mismatch in training and test data.
\end{enumerate*}
Section \ref{sec:main_result} discusses in detail the 
general assumptions on the quantities
$f_\ts$, $F_{\rm out}$, $F_{\rm in}$, and $\phi_{\rm out}$
under which
Theorem~\ref{thm:thm1_complete} holds.

\paragraph*{Prior Work.}
Many recent works 
characterize generalization error of various machine
learning models, including special cases
of the GLM model considered here.
For example, 
the precise characterization for asymptotics of prediction error for least squares regression has been provided in \cite{belkin2019two, hastie2019surprises, muthukumar2019harmless}. The former confirmed the double descent curve of \cite{belkin2019reconciling}
under a Fourier series model and a noisy Gaussian model for data in the over-parameterized regime.
The latter also obtained this scenario under both linear and non-linear feature models for ridge regression and min-norm least squares using random matrix theory. Also, \cite{advani2017high} studied the same setting for deep linear and shallow non-linear networks. 

The analysis of the the generalization for max-margin linear classifiers in the high dimensional regime has been done in \cite{montanari2019generalization}. The exact expression for asymptotic prediction error is derived and in a specific case for two-layer neural network with random first-layer weights, the double descent curve was obtained. A similar double descent curve for logistic regression as well as linear discriminant analysis has been reported by \cite{deng2019model}.
Random feature learning in the same setting has also been studied for ridge regression in \cite{mei2019generalization}. The authors have, in particular, shown that highly over-parametrized estimators with zero training error are statistically optimal at high signal-to-noise ratio (SNR).
The asymptotic performance of regularized logistic regression in high dimensions is studied in \cite{salehi2019impact} using the Convex Gaussian 
Min-max Theorem in the under-parametrized regime. 
The results in the current paper
can consider all these models as special cases.
Bounds on the generalization error of over-parametrized linear models are also given in \cite{bartlett2019benign, neyshabur2018towards}. %

Although this paper and several other recent works
consider only simple linear models and GLMs, 
much of the motivation
is to understand generalization in deep neural networks
where classical intuition
may not hold \cite{belkin2018understand,zhang2016understanding,neyshabur2018towards}.
In particular,
a number of recent papers have shown the connection between neural networks in the over-parametrized regime and kernel methods.  The works \cite{daniely2017sgd, daniely2016toward} showed that gradient descent on over-parametrized neural networks learns a function in the RKHS corresponding to the random feature kernel. Training dynamics of overparametrized neural networks has been studied by \cite{jacot2018neural, du2018gradient, arora2019fine,allen2019learning},
and it is shown that the function learned is in an RKHS corresponding to the neural tangent kernel.

\paragraph*{Approximate Message Passing.}
Our key tool to study the generalization error is
approximate message passing (AMP), a class of inference
algorithms originally developed in \cite{donoho2009message,DonohoMM:10-ITW1,BayatiM:11} for compressed sensing.  We show that the learning problem
for the GLM can be formulated as an inference 
problem on a certain multi-layer network.
Multi-layer AMP methods  \cite{he2017generalized,manoel2018approximate,fletcher2018inference,pandit2019inference} 
can then be applied to perform the inference.  The specific
algorithm we use in this work
is the multi-layer vector AMP (ML-VAMP)
algorithm of \cite{fletcher2018inference,pandit2019inference}
which itself builds on several works 
\cite{opper2005expectation,fletcher2016expectation,rangan2019vamp,cakmak2014samp,ma2017orthogonal}.  The ML-VAMP algorithm
is not necessarily the most computationally efficient
procedure for the minimization \eqref{eq:whatmin}.
For our purposes, the key property
is that ML-VAMP enables exact predictions of
its performance in the large system limit.  
Specifically, the error of the algorithm estimates
in each iteration can be predicted by
a set of deterministic recursive equations called
the \emph{state evolution} or SE\@.
The fixed points
of these equations provide a way of computing the 
asymptotic performance of the algorithm.
In certain cases, the algorithm can be proven to 
be Bayes optimal
\cite{reeves2017additivity,gabrie2018entropy,barbier2019optimal}.

This approach of using AMP methods to characterize the
generalization error of GLMs was also explored in \cite{barbier2019optimal} for i.i.d.\ distributions
on the data. 
The explicit formulae for the asymptotic mean squared error for the regularized linear regression  with rotationally invarient data matrices is proved in \cite{gerbelot2020asymptotic}.
The ML-VAMP method in this
work enables extensions
to correlated features and to mismatch between training and test distributions.

\section{Generalization Error: System Model} \label{sec:lsl}

We consider the problem of estimating the weights
$\wbf$ in the GLM model \eqref{eq:glm}.
As stated in the Introduction, we 
suppose we have training data $\{(\xbf_i,y_i)\}_{i=1}^N$ arranged as $\Xbf:=[\xbf_1 \, \xbf_2 \ldots \xbf_N]\tran \in \R^{N\times p}$,
$\ybf :=[y_1\, y_2\ldots y_N]\tran \in \R^N$.  Then we can write
\beq \label{eq:glm_true}
    \ybf = \phibf_{\rm out}(\Xbf\wbf^0, \dbf),
\eeq
where $\phibf_{\rm out}(\zbf,\dbf)$ is the vector-valued function such that
$[\phibf_{\rm out}(\zbf,\dbf)]_n = \phi_{\rm out}(z_n,d_n)$ and $\{d_n\}_{n=1}^N$ are general noise.

Given the training data $(\Xbf,\ybf)$, we consider 
estimates of $\wbf^0$ given by a 
regularized empirical risk minimization of the form \eqref{eq:whatmin}.
We assume that the loss function $F_{\rm out}$ and regularizer $F_{\rm in}$ are
separable functions, \ie, one can write
\begin{equation}\label{eq:Fsep}
\begin{aligned} 
    F_{\rm out}(\ybf,\zbf) = \sum_{n=1}^N f_{\rm out}(y_n,z_n),
\ \    F_{\rm in}(\wbf) = \sum_{j=1}^p f_{\rm in}(w_j),
\end{aligned}
\end{equation}
for some functions $f_{\rm out}:\Real^{2}\rightarrow\Real$ and $f_{\rm in}:\Real\rightarrow\Real$.
Many standard optimization problems in machine learning can be written in this form:
logistic regression, support vector machines, linear regression, Poisson regression.

\textbf{Large System Limit:} 
We follow the LSL analysis of \cite{BayatiM:11}
commonly used for analyzing AMP-based methods. 
Specifically, we consider a 
sequence of problems indexed by the number of training
samples $N$.  For each $N$, we suppose that the number of features $p=p(N)$
grows linearly with $N$, \ie, 
\begin{align}\label{eq:lsl}\lim_{N \arr \infty} \frac{p(N)}{N} \rightarrow \beta\end{align}
for some constant $\beta\in(0,\infty)$.  Note that $\beta > 1$ corresponds to the over-parameterized regime
and $\beta < 1$ corresponds to the under-parameterized regime. 

\textbf{True parameter:}  
We assume the 
\emph{true} weight vector $\wbf^0$ has components whose
empirical distribution converges as
\beq \label{eq:w0lim}
    \lim_{N \rightarrow \infty} \{ w^0_n \}
    \stackrel{PL(2)}{=} W^0,
\eeq
for some limiting random variable $W^0$.
The precise definition of empirical convergence
is given in Appendix \ref{app:definitions}.
It means that the empirical distribution $\tfrac1p\sum_{i=1}^p\delta_{w_i}$ converges, in the Wasserstein-2 metric
(see Chap. 6 \cite{villani2008optimal}),
to the distribution of the finite-variance random variable $W^0$. 
Importantly, the limit \eqref{eq:w0lim}
will hold if the components $\{w^0_i\}_{i=1}^p$ are drawn i.i.d.\ 
from the distribution of $W^0$ with $\Exp(W^0)^2 < \infty$.  
However, as discussed in Appendix \ref{app:definitions}, the convergence can also be satisfied by correlated sequences and deterministic sequences.

\textbf{Training data input:}
For each $N$,
we assume that the training input data samples, 
$\xbf_i \in \R^p$, $i=1,\ldots,N$,
are i.i.d.\ and drawn from a $p$-dimensional 
Gaussian distribution with zero mean and covariance
$\Sigmabf_{\rm tr} \in \R^{p \times p}$.  
The covariance can capture
the effect of features being correlated.  %
We assume 
the covariance matrix has an eigenvalue decomposition,
\beq \label{eq:Pxtrain_eig}
    \Sigmabf_{\rm tr} = \tfrac{1}{p}\Vbf_0\tran
    \mathrm{diag}(\sbf_{\rm tr}^2)\Vbf_0,
\eeq
where $\sbf_{\rm tr}^2$ are the eigenvalues of $\Sigmabf_{\rm tr}$
and $\Vbf_0 \in \R^{p \times p}$ is the orthogonal
matrix of eigenvectors.  The scaling $\tfrac{1}{p}$
ensures that the total variance of the samples,
$\Exp \|\xbf_i\|^2$,
does not grow with $N$.
We will place a certain random model on $\sbf_{\rm tr}$ and $\Vbf_0$ momentarily.

Using the covariance \eqref{eq:Pxtrain_eig},
we can write the data matrix as
\beq \label{eq:XtrU}
    \Xbf = \Ubf\, \mathrm{diag}(\sbf_{\rm tr})\Vbf_0,
\eeq
where $\Ubf \in \R^{N \times p}$ has entries drawn
i.i.d.\ from ${\mathcal N}(0,\tfrac1p)$.
For the purpose of analysis, it is useful to 
express the matrix $\Ubf$ in terms of its SVD:
\beq \label{eq:Usvd}
    \Ubf = \Vbf_2 \Sbf_{\rm mp}\Vbf_1, \quad
    \Sbf_{\rm mp} := \left[ \begin{array}{cc} 
        \mathrm{diag}(\sbf_{\rm mp}) & \mathbf{0} \\
        \mathbf{0} & * \end{array} \right] 
\eeq
where $\Vbf_1 \in \R^{N \times N}$ and $\Vbf_2 \in \R^{p \times p}$
are orthogonal and $\Sbf_\mp\in\Real^{N\times p}$ with non-zero entries $\sbf_{\rm mp} \in \R^{\min\{N,p\}}$ only along the principal diagonal. $\sbf_\mp$ are the singular values of $\Ubf$. 
A standard result of random matrix theory is that, since $\Ubf$ is i.i.d.\ Gaussian with entries ${\mathcal N}(0,\tfrac{1}{p})$,
the matrices $\Vbf_1$ and $\Vbf_2$ are Haar-distributed on the group of orthogonal matrices and  $\sbf_{\rm mp}$ is such that
\beq \label{eq:smp_lim}
    \lim_{N \rightarrow \infty} \{ s_{{\rm mp},i} \} 
    \stackrel{PL(2)}{=} S_{\rm mp},
\eeq
where $S_{\rm mp} \geq 0$ is a non-negative random variable such that $S_{\rm mp}^2$ satisfies
the Marcencko-Pastur distribution. Details on this distribution are in Appendix \ref{app:mp_dist}.

\textbf{Training data output:}
Given the input data $\Xbf$, we assume that the
training outputs $\ybf$ are generated from 
\eqref{eq:glm_true},
where the noise $\dbf$ is independent of $\Xbf$ and has an empirical 
distribution which converges as
\beq \label{eq:D_lim}
    \lim_{N \rightarrow \infty} \{ d_i \}
    \stackrel{PL(2)}{=} D.
\eeq
Again, the limit \eqref{eq:D_lim} 
will be satisfied if 
$\{d_i\}_{i=1}^N$ are i.i.d. draws of random variable $D$
with bounded second moments.

\textbf{Test data:}
To measure the generalization error, we assume
now that we are given a test point $\xbf_{\rm ts}$,
and we obtain the true output $y_{\rm ts}$ and predicted output $\wh y_{\rm ts}$ given by  \eqref{eq:ytest}.
We assume that the test data inputs are also Gaussian, \ie,
\beq \label{eq:xtest}
    \xbf_{\rm ts}\tran = 
    \ubf\tran\mathrm{diag}(\sbf_{\rm ts})\Vbf_0,
\eeq
where $\ubf \in \R^p$ has 
i.i.d.\ Gaussian components,
${\mathcal N}(0,\tfrac1p)$, and $\sbf_{\rm ts}$
and $\Vbf_0$ are the eigenvalues and eigenvectors
of the test data covariance matrix.  
That is, the test data sample has a covariance matrix
\beq \label{eq:Pxtest_eig}
    \Sigmabf_{\rm ts} = \tfrac{1}{p}\Vbf_0\tran
    \mathrm{diag}(\sbf_{\rm ts}^2)\Vbf_0.
\eeq
In comparison to \eqref{eq:Pxtrain_eig},
we see that we are assuming that the
eigenvectors of the training and test data
are the same, but the eigenvalues may be different.
In this way, we can capture distributional mismatch
between the training and test data.  For example, 
we will be able to measure the generalization error
when the test sample is outside a subspace
explored by the training data.

To capture the relation between the training 
and test distributions, we assume that
components of $\sbf_{\rm tr}$ and $\sbf_{\rm ts}$
converge as
\beq \label{eq:Slim}
    \lim_{N \arr \infty} 
        \{ (s_{{\rm tr},i},s_{{\rm ts},i}) \}
    \stackrel{PL(2)}{=} (S_{\rm tr}, S_{\rm ts}),
\eeq
to some non-negative, bounded 
random vector $(S_{\rm tr}, S_{\rm ts})$.
The joint distribution on $(S_{\rm tr}, S_{\rm ts})$
captures the relation between the training and
test data. 

When
$S_{\rm tr}= S_{\rm ts}$, our model corresponds to
the case when the training and test distribution are matched. Isotropic Gaussian features
in both training and test data  correspond to covariance matrices
$\Sigmabf_{\rm tr} = \tfrac{1}{p}{\sigma^2_{\rm tr}}\Ibf$,
$\Sigmabf_{\rm ts} = \tfrac{1}{p}{\sigma^2_{\rm ts}}\Ibf$,
which can be modeled as
$S_{\rm tr} = \sigma_{\rm tr}$, 
$S_{\rm ts} = \sigma_{\rm ts}$.
We also require that the matrix $\Vbf_0$
is uniformly distributed on the set of $p \times p$
orthogonal matrices.

\textbf{Generalization error:}
From the training data, we obtain an estimate $\wbfhat$ via a regularized empirical risk minimization \eqref{eq:whatmin}.
Given a test sample $\xbf_{\rm ts}$
and parameter estimate $\wbfhat$,
the true output $y_{\rm ts}$ and predicted
output $\wh{y}_{\rm tr}$ are given by equation \eqref{eq:ytest}.
We assume the test noise is distributed as
$d_{\rm ts} \sim D$, 
following the same distribution as the training data. 
The postulated inverse-link function $\phi(\cdot)$ in 
\eqref{eq:ytest} may be different from the true inverse-link
function $\phi_{\rm out}(\cdot)$.

The generalization error is defined as the asymptotic
expected loss, 
\beq \label{eq:Etest}
    {\mathcal E}_{\rm ts} := 
    \lim_{N \rightarrow \infty} 
    \Exp f_{\rm ts}(\wh{y}_{\rm ts},y_{\rm ts}),
\eeq
where $f_{\rm ts}(\cdot)$ is some loss function 
relevant for the test error (which may be different
from the training loss).  
The expectation in \eqref{eq:Etest} is with respect
to the randomness in the training as well as test data, and
the noise. Our main result provides a formula for the generalization error \eqref{eq:Etest}.

\section{Learning GLMs via ML-VAMP} \label{sec:ml_vamp_analysis}

There are many
methods for solving 
the minimization problem \eqref{eq:whatmin}.  
We apply the ML-VAMP algorithm of \cite{fletcher2018inference,pandit2019asymptotics}.
This algorithm is not necessarily the most computationally efficient method.
For our purposes, however, the algorithm serves as a constructive proof technique, \ie, it enables exact predictions for generalization error in the LSL as described above.
Moreover, in the case when loss function \eqref{eq:whatmin} is
strictly convex, the problem has a unique global minimum,
whereby the generalization error of this minimum is agnostic to the choice of algorithm used to find this minimum. To that end, we start by reformulating \eqref{eq:whatmin} in a form that is amicable to the application of ML-VAMP, Algorithm~\ref{algo:ml_vamp}.
\paragraph*{Multi-Layer Representation.}
The first step in applying ML-VAMP to the GLM learning
problem  is to represent the mapping from
the true parameters $\wbf^0$ to the output $\ybf$ as
a certain multi-layer network.  We
combine \eqref{eq:glm_true}, \eqref{eq:XtrU} and \eqref{eq:Usvd},
so that the mapping $\wbf^0\mapsto\ybf$ can be written as the following sequence of operations (as illustrated in Fig.~\ref{fig:graph_mod}):
\begin{equation} \label{eq:sig_flow_true}
\begin{aligned}
    &\zbf^0_0 := \wbf^0, 
    &\pbf^0_0 := \Vbf_0\zbf^0_0, \\
    &\zbf_1^0 := \phibf_1(\pbf^0_0,\xibf_1), 
    &\pbf^0_1 := \Vbf_1\zbf_1^0, \\
    &\zbf_2^0 := \phibf_2(\pbf^0_1,\xibf_2), 
    &\pbf^0_2 := \Vbf_2\zbf_2^0, \\
    &\zbf^0_3 := \phibf_3(\pbf_2^0, \xibf_3)=\ybf,
\end{aligned}
\end{equation}
where $\xibf_\ell$ are the following vectors:
\beq  \label{eq:xidef}
    \xibf_1 := \sbf_{\rm tr}, \quad
    \xibf_2 := \sbf_{\rm mp},
    \quad
    \xibf_3 := \dbf,
\eeq
and the functions $\phibf_\ell(\cdot)$ 
are given by
\begin{equation} \label{eq:phib_ml}
\begin{aligned}
    \phibf_1(\pbf_0,\sbf_{\rm tr}) &:= \mathrm{diag}(\sbf_{\rm tr} )
    \pbf_0, %
    \\
    \phibf_2(\pbf_1,\sbf_{\rm mp}) &:= 
    \Sbf_{\rm mp} \pbf_1,  %
    \\
    \phibf_3(\pbf_2,\dbf) &:= \phibf_{\rm out}(\pbf_2,\dbf).
\end{aligned}
\end{equation}
We see from Fig.~\ref{fig:graph_mod} that the mapping of true parameters
$\wbf^0 = \zbf^0_0$ to the observed response vector
$\ybf  = \zbf^0_3$ 
is described by a multi-layer network of alternating
orthogonal operators $\Vbf_\ell$ and non-linear functions
$\phibf_\ell(\cdot)$.  
Let $L=3$ denote the number of layers in this multi-layer network.

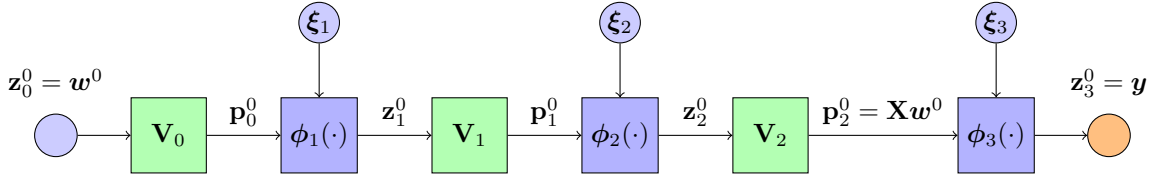
\begin{figure*} 
\centering
\begin{tikzpicture}
   \pgfmathsetmacro{\sep}{2};
    \tikzstyle{var}=[draw,circle,fill=blue!20,node distance=0cm,inner sep=0.1cm];
    \tikzstyle{vary}=[draw,circle,fill=orange!50,node distance=2cm,inner sep=0.1cm];
    
    \tikzstyle{rot}=[draw, fill=green!30, minimum size=1cm, node distance=\sep cm];
    
    \tikzstyle{act}=[draw,fill=blue!30, minimum size=1cm, node distance=\sep cm];
    
    \node [var, inner sep=0.2cm] (z0) {};
    \node [rot, right of=z0,xshift=-.5cm] (V0) {$\Vbf_0$};
    \node [act, right of=V0] (phi1) {$\phibf_1(\cdot)$};
    \node [rot, right of=phi1] (V1) {$\Vbf_1$};
    \node [act, right of=V1] (phi2) {$\phibf_2(\cdot)$};
    \node [rot, right of=phi2] (V2) {$\Vbf_2$};
    \node [act, right of=V2,xshift=1cm]  (phi3) {$\phibf_3(\cdot)$};
    \node [vary, right of=phi3,inner sep=0.2cm,xshift=-.5cm] (y) {};
    \node [above of=y,yshift=-0.3cm] () {$\zbf^0_3=\ybf$};
    \node [above of=z0,yshift=-0.3cm] () {$\zbf^0_0=\wbf^0$};
    
    \draw [->] (z0) --(V0);
    \draw [->] (V0) -- node [above] {$\pbf_0^0$} (phi1);
    \draw [->] (phi1) -- node [above] {$\zbf_1^0$} (V1);
    \draw [->] (V1) -- node [above] {$\pbf_1^0$} (phi2);
    \draw [->] (phi2) -- node [above] {$\zbf_2^0$} (V2);
    \draw [->] (V2) -- node [above,] {$\pbf_2^0 = \Xbf \wbf^0$} (phi3);
    \draw [->] (phi3) -- (y);

    \foreach \i in {1,2,3}
    {
        \node [above of=phi\i,var,inner sep=.3mm,yshift=1.5cm] (xi\i) {$\xibf_\i$};
        \draw [->] (xi\i) -- (phi\i);
    }
\end{tikzpicture} \caption{Sequence flow representing the mapping from the
unknown parameter values $\wbf^0$ to the vector of responses
$\ybf$ on the training data.}
\label{fig:graph_mod}
\end{figure*}

The minimization \eqref{eq:whatmin} can also
be represented using a similar signal flow graph.
Given a parameter candidate $\wbf$, the  mapping $\wbf\mapsto\Xbf\wbf$ can be written using the sequence of vectors
\begin{equation} \label{eq:sig_flow_est}
\begin{aligned}
    &\zbf_0 := \wbf, 
    &\pbf_0 &:= \Vbf_0\zbf_0, \\
    &\zbf_1 := \Sbf_{\rm tr}\pbf_0, 
    &\pbf_1 &:= \Vbf_1\zbf_1, \\
    &\zbf_2 := \Sbf_{\rm mp}\pbf_1, 
    &\pbf_2 &:= \Vbf_2\zbf_2 = \Xbf \wbf.
\end{aligned}
\end{equation}
There are $L=3$ steps in this sequence, and we let 
\[
    \zbf = \{\zbf_0,\zbf_1,\zbf_{2} \}, \quad
    \pbf = \{\pbf_0,\pbf_1,\pbf_{2} \}
\]
denote the sets of vectors across the steps.
The minimization in \eqref{eq:whatmin} can then be written in the following equivalent form:
\beq \label{eq:zpmin1}
\begin{aligned}
    \MoveEqLeft\min_{\zbf,\pbf}\ F_0(\zbf_0) + F_1(\pbf_0,\zbf_1)+    F_2(\pbf_1,\zbf_1) + 
    F_3(\pbf_2)\\
    &{\rm subject\ to}\quad \pbf_\ell = \Vbf_\ell\zbf_\ell, \quad \ell=0,1,2,
\end{aligned}
\eeq
where the \emph{penalty functions} $F_\ell$ are defined as
\begin{equation}
\label{eq:Fell}
\begin{aligned}
    &F_0(\cdot) = F_{\rm in}(\cdot), 
    & F_1(\cdot,\cdot) 
    =& \delta_{\left\{\zbf_1 = \Sbf_{\rm tr}\pbf_0\right\}}(\cdot,\cdot) ,\\
    &F_2(\cdot,\cdot) 
    = \delta_{\left\{\zbf_2 = \Sbf_{\rm mp}\pbf_1\right\}}(\cdot,\cdot),
    & F_3(\cdot) =& F_{\rm out}(\ybf,\cdot),
\end{aligned}
\end{equation}
where $\delta_{\mc A}(\cdot)$ is $0$ on the set $\mc A$, and $+\infty$ on $\mc A^c$.

\paragraph*{ML-VAMP for GLM Learning.}  
Using this multi-layer representation, we can now apply 
the ML-VAMP algorithm from \cite{fletcher2018inference,pandit2019asymptotics}
to solve the optimization \eqref{eq:zpmin1}.
The steps are shown in Algorithm~\ref{algo:ml_vamp}.
These steps are a special case of the  
``MAP version" of ML-VAMP in 
\cite{pandit2019asymptotics}, but with a slightly different
set-up for the GLM problem.
We will call these steps the ML-VAMP GLM Learning Algorithm.

\begin{algorithm}[t]
\caption{ML-VAMP GLM Learning Algorithm}
\begin{algorithmic}[1]  \label{algo:ml_vamp}
\STATE{Initialize  $\gamma^{-}_{0\ell}>0$, $\rbf^-_{0\ell}=0$ for $\ell=0,\ldots,\Lm1$}
    \label{line:init}
\STATE{}
\FOR{$k=0,1,\dots$}\label{line:start_algo_for}
    \STATE{// \texttt{Forward Pass} }
    \FOR{$\ell=0,\ldots,L-1$}
        \IF{$\ell=0$}
           \STATE{$\zbfhat_{k0} = \gbf^+_{0}(\rbf_{k0}^-,\gamma^-_{k0})$}  \label{line:zhatp0}
        \ELSE
        \STATE{$\zbfhat_{k\ell} = \gbf^+_{\ell}(\rbf_{k,\lm1}^+,\rbf_{k\ell}^-,\gamma^+_{k,\lm1},
        \gamma^-_{k\ell})$} 
        \label{line:zhatp}
        \ENDIF
        \STATE{$\alpha_{k\ell}^+= \bkt{\partial \zbfhat_{k\ell}/ \partial \rbf_{k\ell}^-}$} \label{line:alphap}
        \STATE{$\displaystyle \rbf_{k\ell}^+ = 
        \frac{\Vbf_\ell(\zbfhat_{k\ell}-\alpha_{k\ell}^+\rbf_{k\ell}^-)}{1-\alpha^+_{k\ell}}$}
            \label{line:rp}
        \STATE{$\gamma_{k\ell}^+ = (1/\alpha_{k\ell}^+-1)\gamma_{k\ell}^-$}
        \label{line:gamp}
    \ENDFOR
    \STATE{}

    \STATE{// \texttt{Backward Pass} }
    \FOR{$\ell=L,\ldots,1$}
        \IF {$\ell=L$}
        \STATE{$\pbfhat_{k,\Lm1} = \gbf_L^-(\rbf^+_{k,\Lm1},\gamma^+_{k,\Lm1})$} \label{line:phatn0}
        \ELSE
        \STATE{$\pbfhat_{k,\lm1} = \gbf^-_{\ell}(\rbf_{k,\lm1}^+,\rbf_{\kp1,\ell}^-,\gamma^+_{k,\lm1},
        \gamma^-_{\kp1,\ell})$} 
        \label{line:phatn}
        \ENDIF
        \STATE{$\alpha_{k,\lm1}^-=\bkt{\partial \pbfhat_{k,\lm1}/ \partial \rbf_{k,\lm1}^+}$} \label{line:alphan}
        \STATE{$\displaystyle \rbf_{\kp1,\lm1}^- = 
        \frac{\Vbf_{\lm1}\tran(\pbfhat_{k,\lm1}-\alpha_{k,\lm1}^-\rbf_{k,\lm1}^+)}{1-\alpha^-_{k,\lm1}}$}
            \label{line:rn}
        \STATE{$\gamma_{\kp1,\lm1}^- = (1/\alpha_{k,\lm1}^--1)\gamma_{k,\lm1}^+$}
        \label{line:gamn}    
    \ENDFOR

\ENDFOR\label{line:end_algo_for}
\end{algorithmic}
\end{algorithm} 
The algorithm  operates in a set of iterations 
indexed by $k$.  In each iteration, a ``forward pass" through the layers generates estimates 
$\zbfhat_{k\ell}$ for the hidden variables $\zbf^0_\ell$,
while a ``backward pass" generates estimates $\pbfhat_{k\ell}$
for the variables $\pbf^0_\ell$.
In each step, the estimates $\zbfhat_{k\ell}$ and $\pbfhat_{k\ell}$ 
are produced by functions $\gbf^{+}_\ell(\cdot)$
and $\gbf^{-}_\ell(\cdot)$
called \emph{estimators} or \emph{denoisers}.

For the MAP version of ML-VAMP algorithm in \cite{pandit2019asymptotics},
 the denoisers are essentially proximal-type operators defined as
\begin{align}
\label{eq:proximal_operator_def}
\prox_{F/\gamma}(\bm{u}) := \underset{\xbf}{\rm argmin}\ F(\xbf) + \tfrac{\gamma}{2}\norm{\xbf-\bm{u}}^2.
\end{align}
An important property of the proximal operator is that for separable functions $F$ of the form \eqref{eq:Fsep}, we have
$[\prox_{F/\gamma}(\bm{u})]_i=\prox_{f/\gamma}(\bm{u}_i)$. 

In the case of the GLM model,
for $\ell=0$ and $L$, on lines \ref{line:zhatp0} and \ref{line:phatn0}, the denoisers are proximal operators  given by
\begin{subequations} \label{eq:Gb03_def}
\begin{align}
    \gbf_0^+(\rbf_0^-,\gamma_0^-)
     &=\prox_{F_{\rm in}/\gamma_0^-}(\rbf^-_0),
        \label{eq:Gb0_def} \\
    \gbf_3^-(\rbf_2^+,\ybf,\gamma_2^+)
       &=\prox_{F_{\rm out}/\gamma_2^+}(\rbf^+_2).
        \label{eq:Gb3_def}
\end{align}
\end{subequations}
Note that in \eqref{eq:Gb3_def}, there is a dependence
on $\ybf$ through the term $F_{\rm out}(\ybf,\cdot)$.
For the \emph{middle} terms, $\ell=1,2$, \ie, lines \ref{line:zhatp} and \ref{line:phatn}, the denoisers
are given by
\begin{subequations} \label{eq:Gb_pz}
\begin{align}
    &\gbf_\ell^+(\rbf_{\lm1}^+,\rbf_\ell^-,\gamma_{\lm1}^+,\gamma_{\ell}^-) 
        := \zbfhat_\ell, \\
    &\gbf_\ell^-(\rbf_{\lm1}^+,\rbf_\ell^-,\gamma_{\lm1}^+,\gamma_{\ell}^-) 
        := \pbfhat_{\lm1},
\end{align}
\end{subequations}
where $(\pbfhat_{\lm1},\zbfhat_\ell)$ are the solutions to 
the minimization
\begin{align}
    (\pbfhat_{\lm1},\zbfhat_\ell)
        := \underset{(\pbf_{\lm1},\zbf_\ell)}{\rm argmin}\  &F_\ell(\pbf_{\lm1},\zbf_\ell) 
        +  \frac{\gamma_{\ell}^-}{2}\|\zbf_{\ell} - \rbf_{\ell}^-\|^2
        \nonumber \\
        &+\frac{\gamma_{\lm1}^+}{2}\|\pbf_{\lm1} - \rbf_{\lm1}^+\|^2 
        . \label{eq:pz_min}
\end{align}
The quantity $\inner{\partial \bm{v}/\partial \bm{u}}$ on lines \ref{line:alphap} and \ref{line:alphan} denotes the empirical mean $\tfrac1N\sum_{n=1}^N \partial v_n/\partial u_n$.

Thus, the ML-VAMP algorithm in Algorithm~\ref{algo:ml_vamp}
reduces the joint constrained minimization \eqref{eq:zpmin1} 
over variables $(\zbf_0,\zbf_1,\zbf_2)$ and $(\pbf_0,\pbf_1,\pbf_2)$ 
to a set of proximal 
operations on pairs of variables $(\pbf_{\lm1},\zbf_\ell)$.
As discussed in \cite{pandit2019asymptotics}, this type
of minimization is similar to ADMM with adaptive step-sizes.
Details of the denoisers $\gbf_\ell^\pm$ and other aspects of the algorithm
are given in Appendix~\ref{app:ml_vamp_details}.
\section{Main Result} \label{sec:main_result}

We make two assumptions.
The first assumption imposes certain
regularity conditions on the functions
$f_\ts$, $\phi$, $\phi_{\rm out}$,
and maps $\gbf_\ell^\pm$ appearing in Algorithm~\ref{algo:ml_vamp}.
The precise definitions
of pseudo-Lipschitz continuity and
uniform Lipschitz continuity are given
in Appendix~\ref{app:definitions} of the supplementary
material.

\begin{assumption} \label{as:continuity} 
The denoisers and link functions satisfy
the following continuity conditions:
\begin{enumerate}[label=(\emph{\alph*}),topsep=0pt]
    \item The proximal operators in \eqref{eq:Gb03_def},
\[
    \gbf_0^+(\rbf_0^-,\gamma_0^-), \quad
    \gbf_3^-(\rbf_2^+,\ybf,\gamma_2^+),
\]
are uniformly Lipschitz continuous in 
$\rbf^-_0$ and $(\rbf^+_{2},\ybf)$ over parameters 
$\gamma^-_{0}$  and $\gamma^+_{2}$.  

\item The link function $\phi_{\rm out}(p,d)$
is Lipschitz continuous in $(p,d)$.
The test error function 
$f_{\rm ts}(\phi(\zhat),\phi_{\rm out}(z,d))$
is pseduo-Lipschitz continuous in $(\zhat,z,d)$
of order 2.
\end{enumerate}
\end{assumption}

Our second assumption is that the ML-VAMP algorithm converges. 
Specifically,
let $\xbf_k = \xbf_k(N)$ be any set of outputs of Algorithm~\ref{algo:ml_vamp}, 
at some iteration $k$ and dimension $N$.  For example,
$\xbf_k(N)$ could be $\zbfhat_{k\ell}(N)$ or
$\pbfhat_{k\ell}(N)$ for some $\ell$, or a concatenation of signals such as
$\begin{bmatrix}\zbf^0_\ell(N)& \zbfhat_{k\ell}(N)\end{bmatrix}$.

\begin{assumption} \label{as:conv}
Let $\xbf_k(N)$ be any finite set of 
outputs of the ML-VAMP algorithm as above.
Then there exist limits 
\beq \label{eq:xconvk} 
    \xbf(N)=\lim_{k \rightarrow \infty} \xbf_k(N)
\eeq
satisfying
\beq \label{eq:xconv}
    \lim_{k \rightarrow \infty} \lim_{N \rightarrow \infty}
        \frac{1}{N} \|\xbf_k(N) - \xbf(N)\|^2  = 0.
\eeq
\end{assumption}

We are now ready to state our main result.

\begin{theorem} \label{thm:thm1_complete} 
Consider the GLM learning problem \eqref{eq:whatmin} solved by applying Algorithm~\ref{algo:ml_vamp} to the equivalent problem \eqref{eq:zpmin1}
under the assumptions of Section~\ref{sec:lsl}
along with Assumptions~\ref{as:continuity} and~\ref{as:conv}. 
Then, there exist constants $\tau_0^-,\gammabar_0^+>0$ and $\Mbf\in \Real^{2\times 2}_{\succ 0}$ such that the following hold:
\ignore{Let $(\zbfhat,\pbfhat)$ be a fixed point of Algorithm \ref{algo:ml_vamp}, and $(\Kbf_0^+, \tau_0^-, \gammabar_0^+,\gammabar_1^-, \tau_1^-)$ be the unique fixed point of Algorithm \ref{algo:ml_vamp_se} such that 
$\Kbf_0^+=\begin{bmatrix}k_{11}& k_{12}\\ k_{12} & k_{22}\end{bmatrix}\succ 0$ and 
$\tau_0^-, \gammabar_0^+,\gammabar_1^-, \tau_1^->0$.
Then, we have
}
\begin{enumerate}[label=(\alph*)]
\item The fixed points $\{\zbfhat_\ell,\pbfhat_\ell\}$,
$\ell=0,1,2$ of Algorithm \ref{algo:ml_vamp} satisfy the KKT conditions for the constrained optimization problem \eqref{eq:zpmin1}. Equivalently $\wbfhat:=\zbfhat_0$ is a stationary point of \eqref{eq:whatmin}.

\item %
The true parameter $\wbf^0$ and its 
estimate $\wbfhat$ empirically converge as
\beq \label{eq:wwhat_lim}
    \lim_{N \rightarrow \infty}
    \{ (w^0_i,\wh{w}_{i}) \} 
    \stackrel{PL(2)}{=} (W^0,\What),
\eeq
where $W^0$ is the random variable from \eqref{eq:w0lim} and
\beq \label{eq:wwhat_mod}
    \What = \prox_{f_{\rm in}/\gammabar_0^+}(W^0+  Q_0^-),
\eeq
with $Q_0^- = \Norm(0,\tau_0^-)$ independent of $W^0$.

\item 
The asymptotic generalization error \eqref{eq:Etest} with $(y_\ts,\wh y_\ts)$  defined as \eqref{eq:ytest}
is given by
\begin{align}\label{eq:generalization_main_result}
    \mc E_{\rm ts} = \Exp\,f_{\rm ts}\!\left(\phi_{\rm out}(Z_{\rm ts},D),\phi(\Zhat_{\rm ts})\right),
\end{align}
where $(Z_{\rm ts},\Zhat_{\rm ts})\sim \mc N(\zero_2,\Mbf)$ and independent of $D$. 
\ignore{
Further, the entries of covariance $\Mbf\in\Real^{2\times 2}$ are 
\begin{itemize}
    \item $m_{11} = \Exp S_\ts^2\cdot k_{11}$,
    \item $m_{12} = m_{11}+ \Exp\left(\frac{S_\ts^2\wb\gamma_0^+}{\wb\gamma_0^++S_\tr^2\wb\gamma_1^-}\right)\cdot k_{12}$, and
    \item 
    $m_{22}= \Exp\left(\frac{\wb\gamma_0^+ S_\ts}{\wb\gamma_0^++S_\tr^2\wb\gamma_1^-}\right)^2 k_{22}
     \nonumber\\
     +\Exp  \left(\frac{\wb\gamma_1^- S_\tr S_\ts}{\wb\gamma_0^++S_\tr^2\wb\gamma_1^-}\right)^2\tau_1^- + 2m_{12}-m_{11} .
    $
\end{itemize}
}
\end{enumerate}
\end{theorem}

Part (a) shows that, similar to gradient descent, Algorithm \ref{algo:ml_vamp} finds the stationary points of problem \eqref{eq:whatmin}.
These stationary points will be unique in strictly convex problems such as linear and logistic regression.
Thus, in such cases, 
the same results will be true for any 
algorithm that finds such stationary points.
Hence, the fact that we are using ML-VAMP is immaterial
-- our results apply to any solver for \eqref{eq:whatmin}.
Note that the convergence to the
fixed points $\{\zbfhat_\ell,\pbfhat_\ell\}$
is assumed from Assumption~\ref{as:conv}.

Part (b) provides an exact description
of the asymptotic 
statistical relation between the
true parameter $\wbf^0$ and its
estimate $\wh \wbf$.  
The parameters $\tau_0^-,\gammabar_0^+>0$ and $\Mbf$
can be explicitly computed using a set
of recursive equations called the \emph{state
evolution} or SE described in Appendix \ref{app:SE} in the
supplementary material.

We can use the expressions 
to compute a variety of relevant metrics.
For example, the $PL(2)$ convergence 
shows that the MSE on the parameter
estimate is
\beq \label{eq:wmse}
    \lim_{N \rightarrow \infty}
    \frac{1}{N} \sum_{n=1}^N
        (w_n^0-\what_n)^2 = 
    \Exp(W^0-\What)^2.
\eeq
The expectation on the right hand side
of \eqref{eq:wmse} can then be computed
via integration over the joint density
of $(W^0,\What)$ from part (b).  
In this way, we have a simple and exact
method to compute the parameter error.
Other metrics such as parameter 
bias or variance, cosine angle or sparsity
detection can also be computed.

Part (c) of Theorem \ref{thm:thm1_complete}
similarly exactly characterizes the asymptotic generalization error. In this case, 
we would compute the expectation over 
the three variables $(Z,\Zhat,D)$.
In this way, we have provided a methodology
for exactly predicting the
generalization error from 
the key parameters of the problems
such as the sampling ratio
$\beta=p/N$,
the regularizer, 
the output function, 
and the distributions of the true weights and noise.
We provide several
examples such as linear regression, logistic regression and SVM in the Appendix \ref{app:special_cases}.
We also recover the result by \cite{hastie2019surprises} in Appendix \ref{app:special_cases}.

\paragraph{Remarks on Assumptions.}
Note that Assumption~\ref{as:continuity} is satisfied 
in many practical cases.
For example, it can be verified that it is satisfied
in the case when $f_{\rm in}(\cdot)$ and $f_{\rm out}(\cdot)$
are convex. 
Assumption \ref{as:conv} is somewhat more restrictive in
that it requires that the ML-VAMP algorithm converges.
The convergence properties of ML-VAMP are discussed in
\cite{fletcher2016expectation}. 
The ML-VAMP algorithm may not always converge, and characterizing conditions under which convergence is possible is an open question.
However, experiments in \cite{rangan2019vamp} show that
the algorithm does indeed often converge,
and in these cases, our analysis applies.
In any case, we will see below that the predictions
from Theorem~\ref{thm:thm1_complete} agree closely
with numerical experiments in several relevant cases.

In some special cases equation \eqref{eq:generalization_main_result} simplifies to yield quantitative insights for interesting modeling artifacts.
We discuss these in Appendix \ref{app:special_cases} in the supplementary material.

\section{Experiments}

\paragraph*{Training and Test Distributions.}
We validate our theoretical results on a number of synthetic
data experiments.  For all the experiments,  
the training and test data is generated 
following the model in 
Section~\ref{sec:lsl}.  We generate the training
and test eigenvalues as i.i.d.\ with lognormal distributions,
\[
    S_{\rm tr}^2 = A (10)^{0.1 u_{\rm tr}}, \quad
    S_{\rm ts}^2 = A (10)^{0.1 u_{\rm ts}}, 
\]
where $(u_{\rm tr},u_{\rm ts})$ are bivariate zero-mean Gaussian with 
\[
    \mathrm{cov}(u_{\rm tr},u_{\rm ts}) = \sigma^2_u\left[
        \begin{array}{cc} 1 & \rho \\ \rho & 1 \end{array} 
        \right].
\]
In the case when $\sigma^2_u=0$, we obtain eigenvalues
that are equal, corresponding to the i.i.d.\ case.
With $\sigma^2_u > 0$ we can model correlated features.
Also, when the correlation coefficient $\rho=1$, $S_{\rm tr}=S_{\rm ts}$, so there is
no training and test mismatch.   
However, we can also select $\rho < 1$ to
experiment with cases when 
the training and test distributions differ.
In the examples below, we consider the following three cases:
\begin{enumerate}[label=(\arabic*),parsep=0pt,topsep=0pt]
\item i.i.d.\ features ($\sigma_u=0$);
\item correlated features with matching training and test 
distributions ($\sigma_u = 3$~dB,
$\rho=1$); and
\item correlated features with train-test mismatch ($\sigma_u = 3$~dB,
$\rho=0.5$).
\end{enumerate}
For all experiments below, the true model coefficients are generated as i.i.d.\ 
Gaussian  $w_j^0 \sim \Norm(0,1)$ and we use standard L2-regularization,
$f_{\rm in}(w)=\lambda w^2/2$ for some $\lambda > 0$.
Our framework can incorporate arbitrary i.i.d.\ distributions on $w_j$
and regularizers, but we will illustrate just the Gaussian case 
with L2-regularization here.

\paragraph*{Under-regularized linear regression.}
We first consider the case of under-regularized
linear regression where the output channel is $\phi_{\rm out}(p,d) = p + d$
with $d \sim \Norm(0,\sigma^2_d)$.  The noise variance $\sigma^2_d$ is set
for an SNR level of 10~dB.  We use a standard mean-square error (MSE) output loss, 
$f_{\rm out}(y,p)= (y-p)^2/(2\sigma^2_d)$.  Since we are using the 
L2-regularizer, $f_{\rm in}(w)=\lambda w^2/2$, the minimization \eqref{eq:whatmin}
is standard ridge regression.  
Moreover, if we were to select $\lambda = 1/\Exp(w_j^0)^2$,
then the ridge regression estimate would correspond to the minimum mean-squared error (MMSE) estimate
of the coefficients $\wbf^0$.
However, to study the under-regularized regime,
we take $\lambda = (10)^{-4}/\Exp(w_j^0)^2$.

Fig.~\ref{fig:linear} plots the test MSE 
for the three cases described above
for the linear model.
In the figure, we take $p=1000$ features and vary the number of samples $n$ from $0.2p$
(over-parametrized) to $3p$ (under-paramertrized).
For each value of $n$,
we take 100 random instances of the model and compute the ridge regression 
estimate using the sklearn package and measure the test MSE
on the 1000 independent test samples.  The simulated values in Fig.~\ref{fig:linear}
are the median test error over the 100 random trials.  The test MSE is plotted in a
normalized dB scale,
\[
    \mbox{Test MSE (dB)}=
    10\log_{10}\left( 
    \frac{\Exp(\wh{y}_{\rm ts}-y_{\rm ts})^2}
    {\Exp y_{\rm ts}^2} \right).
\]
Also plotted is the 
state evolution (SE) theoretical test MSE from Theorem~\ref{thm:thm1_complete}.

\begin{figure}[!t]
    \centering
    \includegraphics[width=0.5\textwidth]{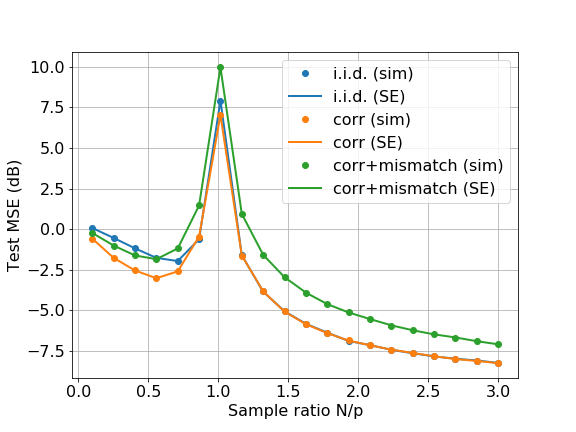}
    \caption{Test error for under-regularized
    linear regression under various train and test distributions}
    \label{fig:linear}
\end{figure}

In all three cases in Fig.~\ref{fig:linear}, the SE theory exactly matches
the simulated values for the test MSE\@.
Note that the case of match training and test 
distributions for this problem was studied in \cite{hastie2019surprises,mei2019generalization,montanari2019generalization}
and we see the \emph{double descent} phenomenon described in their work.
Specifically, with highly under-regularized linear regression,
the test MSE actually \emph{increases} with more samples $n$ in the over-parametrized
regime ($n/p < 1$) and then decreases again in the under-parametrized regime
($n/p > 1$).  

Our SE theory can also provide predictions
for the correlated feature case.
In this particular setting, 
we see that in the correlated case the 
test error is slightly lower in the over-parametrized regime since the
energy of data is concentrated in a smaller sub-space.  Interestingly,
there is minimal difference between the correlated and i.i.d.\ cases
for the under-parametrized regime when the training and test data match.
When the training and test data are not matched, the test error increases.
In all cases, the SE theory can accurately predict these effects.

\begin{figure}[!t]
    \centering
    \includegraphics[width=0.5\textwidth]{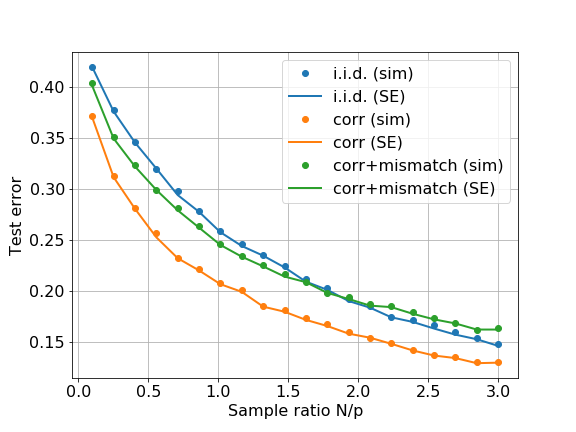}
    \caption{Classification
    error rate with
    logistic regression under various train and test distributions}
    \label{fig:logistic}
\end{figure}

\paragraph*{Logistic Regression.}
Fig.~\ref{fig:logistic} shows a similar plot as Fig.~\ref{fig:linear}
for a logistic model.
Specifically, we use a logistic output
$P(y=1)=1/(1+e^{-p})$,
a binary cross entropy output loss 
$f_{\rm out}(y,p)$, and
$\ell_2$-regularization level $\lambda$ so that the
output corresponds to the MAP estimate
(we do not perform ridgeless regression in this case).
The mean of the training and test eigenvalues
$\Exp S_{\rm tr}^2=\Exp S_{\rm ts}^2$
are selected such that, if the true coefficients $\wbf^0$ were known,
we could obtain a 5\% prediction error.
As in the linear case,
we generate random instances of the model,
use the sklearn package to perform the logistic regression, and
evaluate the estimates on 1000 new test samples.
We compute the median error rate ($1-$ accuracy) and compare 
the simulated values with the SE theoretical estimates.
The i.i.d.\ case was considered in \cite{salehi2019impact}.
Fig.~\ref{fig:logistic} shows that our SE theory is able to predict
the test error rate exactly in i.i.d.\ cases along with a correlated
case and a case with training and test mismatch.

\begin{figure}
    \centering
    \includegraphics[width=0.5\textwidth]{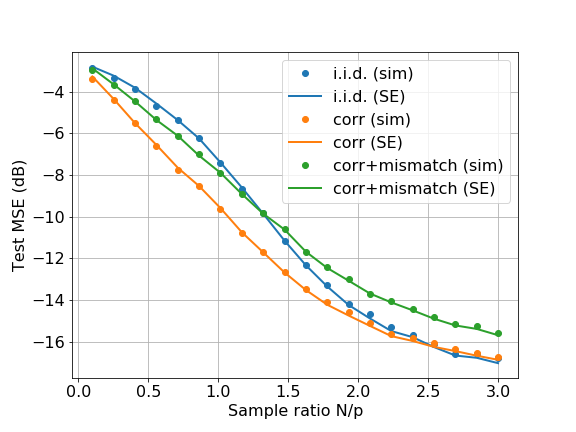}
    \caption{Test MSE under a non-linear least square estimation.}
    \label{fig:nonlin}
\end{figure}

\paragraph*{Nonlinear Regression.}
The SE framework can also consider non-convex problems.
As an example, we consider a non-linear regression problem where the 
output function is
\beq \label{eq:out_tanh}
    \phi_{\rm out}(p,d) = \tanh(p) + d, 
    \quad d \sim \Norm(0,\sigma^2_d).
\eeq
The $\tanh(p)$ models saturation in the output.  Corresponding to this output,
we use a non-linear MSE output loss
\beq \label{eq:fout_tanh}
    f_{\rm out}(y-p) = \frac{1}{2\sigma^2_d}(y-\tanh(p))^2. 
\eeq
This output loss is non-convex.  
We scale the 
data matrix so that the input $\Exp(p^2)=9$ so that the $\tanh(p)$
is driven well into the non-linear regime.  We also take $\sigma^2_d = 0.01$.

For the simulation, the non-convex loss is minimized using Tensorflow 
where the non-linear model is described as a two-layer model.  
We use the ADAM optimizer \cite{kingma2014adam} with 200 epochs
to approach a local minimum of the objective \eqref{eq:whatmin}.
Fig.~\ref{fig:nonlin} plots the median test MSE for the estimate 
along with the SE theoretical test MSE\@.
We again see that the SE theory is able to predict
the test MSE in all cases even for this non-convex problem. %
\section{Conclusions}

In this paper we provide a procedure
for exactly computing the asymptotic generalization error of a solution in a generalized linear model (GLM).
This procedure is based on scalar quantities which are fixed points of a recursive iteration.
The formula holds for a large class of generalization metrics, loss functions, and regularization schemes.
Our formula allows analysis of important modeling effects such as \begin{enumerate*}[label=(\roman*)]
    \item overparameterization,
    \item dependence between covariates, and
    \item mismatch between train and test distributions,
\end{enumerate*}
which play a significant role in the analysis and design of machine learning systems. We experimentally validate our theoretical results for linear as well as non-linear regression and logistic regression, where a strong agreement is seen between our formula and simulated results.

\bibliography{ref}
\bibliographystyle{icml2020.bst}

\clearpage
\appendix
\icmltitlerunning{Supplementary Materials: Generalization Error of GLMs in High Dimensions}
\section{Empirical Convergence of Vector Sequences}
\label{app:definitions}

The LSL model in Section~\ref{sec:lsl} and our main
result in Section~\ref{sec:main_result} require 
certain technical definitions.

\begin{definition}[Pseudo-Lipschitz continuity]
\label{def:pl_cont}
For a given $p\geq 1$, a function $\fbf:\Real^{d}\rightarrow \Real^{m}$ is called Pseudo-Lipschitz of order $p$ if
\begin{align}
    \MoveEqLeft \|\fbf(\xbf_1)-\fbf(\xbf_2)\| \nonumber \\&
    \leq C\|\xbf_1-\xbf_2\|\left(1+\|\xbf_1\|^{p-1}+\|\xbf_2\|^{p-1}\right)
\end{align}
for some constant $C > 0$.  
\end{definition}

Observe that for $p=1$, the pseudo-Lipschitz is equivalent to
the standard definition of Lipschitz continuity.

\begin{definition}[Uniform Lipschitz continuity] Let $\phibf(\xbf,\theta)$ be a function on $\rbf \in \R^d$ and $\theta \in \R^s$.
We say that $\phibf(\xbf,\theta)$ is \emph{uniformly Lipschitz continuous} in $\xbf$ at $\theta=\thetabar$ if there exists constants $L_1,L_2 \geq 0$ and an open neighborhood $U$ of $\thetabar$ such that
\begin{align} \label{eq:unifLip1}
    \|\phibf(\xbf_1,\theta)-\phibf(\xbf_2,\theta)\| \leq L_1\|\xbf_1-\xbf_2\|
\end{align}
for all
$\xbf_1,\xbf_2 \in \R^d$ and $\theta \in U$; and
\begin{align}
    \|\phibf(\xbf,\theta_1)-\phibf(\xbf,\theta_2)\| \leq L_2\left(1+\|\xbf\|\right)
    \|\theta_1-\theta_2\|,
\end{align}
for all $\xbf \in \R^d$ and $\theta_1,\theta_2 \in U$.
\end{definition}

\newcommand{\PLT}{\rm PL(2)}
\begin{definition}[Empirical convergence of a sequence]
\label{def:plp_convergence} Consider a sequence of vectors $\xbf(N) = \{\xbf_n{(N)}\}_{n=1}^N$ with $\xbf_n(N)\in\Real^{d}$. So, each $\xbf(N)$ is a block vector
with a total of $Nd$ components.  
For a finite $p\geq 1$, we say that the vector sequence $\xbf(N)$ converges empirically with $p$-th order moments if there exists a random variable $X \in \Real^d$ such that
\begin{enumerate}[label=(\roman*)]
\item $\Exp\|X\|_p^p < \infty$; and
\item for any $f : \R^d \arr \R$ that is pseudo-Lipschitz continuous of order $p$,
\beq \label{eq:PLp-empirical}
    \lim_{N \arr \infty} \tfrac{1}{N} \sum_{n=1}^N f(\xbf_n(N)) = \Exp\left[ f(X) \right].
\eeq
\end{enumerate}
\end{definition}

In this case, with some abuse of notation, we will write
\beq \label{eq:plLim}
    \lim_{N \arr \infty} \left\{ \xbf_n \right\} \stackrel{PL(p)}{=} X,
\eeq
where we have omitted the dependence on $N$ in $\xbf_n(N)$.
We note that the sequence $\{\xbf{(N)}\}$ can be random or deterministic. If it is random, we will require that
for every pseudo-Lipschitz function $f(\cdot)$,
the limit \eqref{eq:PLp-empirical} holds almost surely.
In particular, if $\xbf_n \sim X$ are i.i.d.\ and $\Exp\|X\|^p_p < \infty$, then $\xbf$ empirically converges to $X$ with $p^{\rm th}$ order moments. 

$PL(p)$ convergence is equivalent to weak convergence plus convergence in $p^{\rm}$ moment \cite{BayatiM:11}, and hence $PL(p)$ convergence is also equivalent to convergence in Wasserstein-$p$ metric (See Chapter 6. \cite{villani2008optimal}). We use this fact later in proving Theorem \ref{thm:thm1_complete}. %
\section{ML-VAMP Denoisers Details} \label{app:ml_vamp_details}

Related to $\Sbf_\mp$ and $\sbf_\mp$ from equation \eqref{eq:Usvd}, we need to define two quantities $\sbf_\mp^+ \in\Real^N$ and $\sbf_\mp^- \in \Real^p$ that are zero-padded versions of the singular values $\sbf_{\mp}$, so that for $n > \min\{N,p\}$,
we set $s^\pm_{{\rm mp},n}=0$.
Observe that $(\sbf_\mp^+)^2$ are eigenvalues of $\Ubf\Ubf\tran$ whereas $(\sbf_\mp^-)^2$ are eigenvalues of $\Ubf\tran\Ubf$. Since $\sbf_\mp$ empirically converges to $S_\mp$ as given in \eqref{eq:smp_lim}, 
the vector $\sbf_\mp^+$ empirically converges to random variable $S_\mp^+$ whereas the vector $\sbf_\mp^-$ empirically converges to random variable $S_\mp^-$, where a mass is placed at $0$ appropriately. Specifically, $S_\mp^+$ has a point mass of $(1-\beta)_+\delta_{\{0\}}$ when $\beta<1$, whereas $S_\mp^-$ has a point mass of $(1-\frac1\beta)_+\delta_{\{0\}},$ when $\beta>1.$  
In Appendix \ref{app:mp_dist} (eqn.  \eqref{eq:MP_density}), we provide the densities over positive parts of $S_\mp^+$ and $S_\mp^-$.

A key property of our analysis will be that 
the non-linear functions \eqref{eq:phib_ml}
and the denoisers $\gbf_\ell^{\pm}(\cdot)$
have simple forms.
 
\noindent
\underline{Non-linear functions $\phibf_\ell(\cdot)$:}
The non-linear functions all 
act \emph{componentwise}.  
For example, for $\phibf_1(\cdot)$ in \eqref{eq:phib_ml},
we have
\begin{align*}
    \MoveEqLeft 
    \zbf_1 = \phibf_1(\pbf_0,\sbf_{\rm tr})  = 
    \mathrm{diag}(\sbf_{\rm tr})\pbf_0
     \Longleftrightarrow
    z_{1,n} = \phi_1(p_{0,n},s_{{\rm tr},n}),
\end{align*}
where $\phi_1(\cdot)$ is the scalar-valued function,
\beq \label{eq:phi1_def}
    \phi_1(p_0,s) = sp_0. 
\eeq
Similarly, for $\phibf_2(\cdot)$,
\begin{align*}
    \MoveEqLeft 
    \zbf_2 = \phibf_2(\pbf_1,\sbf^+_{\rm mp})  
    \Longleftrightarrow
    z_{2,n} = \phi_2(\wb{p}_{1,n},{s}^+_{{\rm mp},n}),\quad n<N
\end{align*}
where $\wb \pbf_1 \in\Real^N$ is the zero-padded version of $\pbf_1$, and
\beq \label{eq:phi2_def}
    \phi_2(p_1,s) = s\, p_1.
\eeq
Finally, the function $\phibf_3(\cdot)$
in \eqref{eq:phib_ml} acts componentwise with 
\beq \label{eq:phi3_def}
    \phi_3(p_2,d) = \phi_{\rm out}(p_2,d).
\eeq

\noindent
\underline{Input denoiser $\gbf_0^+(\cdot)$:}
Since $F_0(\zbf_0) = F_{\rm in}(\zbf_0)$, 
and $F_{\rm in}(\cdot)$ given in \eqref{eq:Fsep},
the denoiser \eqref{eq:Gb0_def} acts 
\emph{componentwise} in that,
\[
    \zbfhat_0 = \gbf_0^+(\rbf_0^-,\gamma_0^-) 
    \Longleftrightarrow
    \zhat_{0,n} = g_0^+(r_{0,n}^-,\gamma_0^-),
\]
where $g_0^+(\cdot)$ is the scalar-valued function,
\beq \label{eq:G0def_app}
    g_0^+(r_0^-,\gamma_0^-) %
        := \argmin_{z_0}\ f_{\rm in}(z_0) 
        + \frac{\gamma_0^-}{2}(z_0 - r_0^-)^2.
\eeq
Thus, the vector optimization in \eqref{eq:Gb0_def}
reduces to a set of scalar optimizations \eqref{eq:G0def_app}
on each component.  

\noindent
\underline{Output denoiser $\gbf_3^-(\cdot)$:}
The output penalty $F_3(\pbf_2,\ybf) = F_{\rm out}(\pbf_2,\ybf)$
where $F_{\rm out}(\pbf_2,\ybf)$ has the separable
form \eqref{eq:Fsep}.  Thus, similar to the case
of $\gbf_0(\cdot)$,
the denoiser $\gbf_3(\cdot)$
in \eqref{eq:Gb3_def} also acts componentwise with
the function,
\begin{align}
    \MoveEqLeft g_3^-(r_2^+,\gamma_2^+,y):= \argmin_{p_2}\  f_{\rm out}(p_2,y)  
        + \tfrac{\gamma_2^+}{2}(p_2 - r_2^+)^2.
        \label{eq:G3def}
\end{align}

\noindent
\underline{Linear denoiser $\gbf_1^{\pm}(\cdot)$:} The expressions for both denoisers $g_1^\pm$ and $g_2^\pm$ are very similar and can be explained together.
The penalty $F_1(\cdot)$ restricts $\zbf_1 = \Sbf_{\rm tr}\pbf_0$, where $\Sbf_\tr$ is a square matrix.
Hence, for $\ell=1,$ the minimization in
\eqref{eq:pz_min} is given by,
\begin{align}
    \MoveEqLeft \pbfhat_{0}
        := \argmin_{\pbf_{0}}\   \tfrac{\gamma_{0}^+}{2}\|\pbf_{0} - \rbf_{0}^+\|^2 %
       + \tfrac{\gamma_{1}^-}{2}\|\Sbf_{\rm tr}\pbf_{0} - \rbf_{1}^-\|^2
        , \label{eq:pmin3}
\end{align} 
and $\zbfhat_1 = \Sbf_{\rm tr}\pbfhat_{0}$.
This is a simple quadratic minimization and the
components of $\pbfhat_{0}$ and $\zbfhat_1$ are given
by
\begin{align}  
    \phat_{0,n} &= g_1^-(r_{0,n}^+,r_{1,n}^-,\gamma_{0}^+,
    \gamma^-_1,s_{{\rm tr},n}) \nonumber \\
    \zhat_{1,n} &= g_1^+(r_{0,n}^+,r_{1,n}^-,\gamma_{0}^+,
    \gamma^-_1,s_{{\rm tr},n}), \nonumber
\end{align}
where
\begin{subequations} 
\begin{align} \label{eq:G23def}
     g_1^-(r_{0}^+,r_{1}^-,\gamma_{0}^+,
    \gamma^-_1,s)
    &:= 
    \frac{\gamma^+_{0}r^+_{0} + s\gamma^-_{1}r^-_{1}}{\gamma^+_{0} + s^2\gamma^-_{1}}
    \\
     g_1^+(r_{0}^+,r_{1}^-,\gamma_{0}^+,
    \gamma^-_1,s)
    &:= 
    \frac{s(\gamma^+_{0}r^+_{0} + s\gamma^-_{1}r^-_{1})}{\gamma^+_{0} + s^2\gamma^-_{1}}\label{eq:g_lp}
\end{align}
\end{subequations}

\underline{Linear denoiser $\gbf_2^{\pm}(\cdot)$:}
This denoiser is identical to the case $\gbf_1^{\pm}(\cdot)$
in that we need to impose the linear constraint $\zbf_2 = \Sbf_{\rm mp}\pbf_1$. However $\Sbf_\mp$ is in general a rectangular matrix and the two resulting cases of $\beta\lessgtr1$ needs to be treated separately. 

Recall the definitions of vectors $\sbf_\mp^+$ and $\sbf_\mp^-$ at the beginning of this section. 
Then, for $\ell=2$,
with the penalty $F_2(\pbf_1,\zbf_2)=\delta_{\{\zbf_2=\Sbf_{\rm mp}\pbf_1\}}$, 
the solution to \eqref{eq:pz_min} has components,
\begin{subequations}
\begin{align}  
    \phat_{1,n} &= g_2^-(r_{1,n}^+,r_{2,n}^-,\gamma_{1}^+,
    \gamma^+_2,s^-_{{\rm mp},n}) \label{eq:p2G} \\
    \zhat_{2,n} &= g_2^+(r_{1,n}^+,r_{2,n}^-,\gamma_{1}^+,
    \gamma^+_2,s^+_{{\rm mp},n}), \label{eq:z3G}
\end{align}
\end{subequations}
with the identical functions $g_2^-=g_1^-$ and $g_2^+=g_1^+$ as given by \eqref{eq:G23def} and \eqref{eq:g_lp}.
Note that in \eqref{eq:p2G}, $n=1,\ldots,p$
and in \eqref{eq:z3G}, $n=1,\ldots,N$.

\section{State Evolution Analysis of ML-VAMP}
\label{app:SE}
A key property of the ML-VAMP algorithm 
is that its 
performance  in the LSL can be exactly 
described by a \emph{scalar equivalent system}. 
In the scalar equivalent system, the vector-valued outputs 
of the algorithm  are replaced by 
scalar random variables representing the
typical behavior of the components of the vectors
in the large-scale-limit (LSL).
Each of the random variables are described by a set
of  parameters, where the parameters are given
by a set of deterministic equations called the 
\emph{state evolution} or SE.  

\begin{algorithm}
\caption{SE for ML-VAMP for GLM Learning}
\begin{algorithmic}[1]  \label{algo:ml_vamp_se}

\STATE{// \texttt{Initial} }
\STATE{Initialize $\gammabar^{-}_{0\ell}=\gamma_{0\ell}^-$ from Algorithm \ref{algo:ml_vamp}.}
\STATE{$Q^-_{0\ell}\sim \mc N(0,\tau_{0\ell}^-)$ for some $\tau_{0\ell}^->0$ for $\ell=0,1,2$}  \label{line:init_se}
\STATE{$Z^0_0 = W^0$} \label{line:z0init}
\FOR{$\ell=0,\ldots,\Lm1$}
    \STATE{$P^0_\ell = \Norm(0,\tau^0_\ell), \quad \tau^0_\ell = \mathrm{var}(Z^0_\ell)$ }
    \STATE{$Z^0_{\lp1} = \phi_{\lp1}(P^0_{\ell},\Xi_{\lp1})$ }
      \label{line:zinit_se}
\ENDFOR
\STATE{}
\FOR{$k=0,1,\dots$}
    \STATE{// \texttt{Forward Pass} }
    \FOR{$\ell=0,\ldots,L-1$}
        \IF{$\ell=0$}
        \STATE{$R^-_{k0}=Z^0_{\ell}+Q^-_{k0}$}
        \label{line:rp0_se}
        \STATE{$\Zhat_{k0} = g^+_{0}(R^-_{k0},\gammabar^-_{k0})$}  \label{line:zhatp0_se}
        \ELSE
         \STATE{$R^+_{k,\lm1}=P^0_{\lm1}+P^+_{k,\lm1}$,
      $R^-_{k\ell}=Z^0_{\ell}+Q^-_{k\ell}$}
        \label{line:rp_se}
         \STATE{$\Zhat_{k\ell} = g^+_{\ell}(R^+_{k,\lm1},R^-_{k\ell},\gammabar^+_{k,\lm1},\gammabar^-_{k\ell},\Xi_\ell)$}  \label{line:zhatp_se}
        \ENDIF
    \STATE{$\alphabar_{k\ell}^+=\Exp\partial \Zhat_{k\ell}/ \partial Q_{k\ell}^-$} \label{line:alphap_se}
    \STATE{$\displaystyle Q_{k\ell}^+ =
        \frac{\Zhat_{k\ell}-Z^0_\ell-\alphabar_{k\ell}^+Q_{k\ell}^-}{1-\alphabar^+_{k\ell}}$}  \label{line:qp_se}
    \STATE{$\gammabar_{k\ell}^+ = (\tfrac{1}{\alphabar_{k\ell}^+}-1)\gammabar_{k\ell}^-$}
        \label{line:gamp_se}
    \STATE{$(P^0_\ell,P^+_{k\ell}) \sim \Norm(0,\Kbf_{k\ell}^+), ~
    \Kbf_{k\ell}^+ = \mathrm{cov}(Z^0_\ell,Q^+_{k\ell})$} \label{line:Kp_se}
    \ENDFOR
    \STATE{}

    \STATE{// \texttt{Backward Pass} }
 
    \FOR{$\ell=L,\ldots,1$}
       \IF{$\ell=L$}
     \STATE{$R^+_{k,\Lm1}=P^0_{\Lm1}+P^+_{k,\Lm1}$}
        \label{line:rn0_se}
       \STATE{$\Phat_{k,\Lm1} = g_L^-(R^+_{k,\Lm1},\gammabar^+_{k,\Lm1},
        Z^0_L)$} \label{line:phatn0_se}
        \ELSE
      \STATE{$R^+_{k,\lm1}=P^0_{\lm1}+P^+_{k,\lm1}$,
      $R^-_{\kp1,\ell}=Z^0_{\ell}+Q^-_{\kp1,\ell}$}
        \label{line:rn_se}
      \STATE{$\Phat_{k,\lm1} = g_\ell^-(R^+_{k,\lm1},R^-_{\kp1,\ell},
        \gammabar^+_{k,\lm1},\gammabar^-_{\kp1,\ell},\Xi_{\ell})$} \label{line:phatn_se}
        \label{line:phatn_se}
        \ENDIF
        \STATE{$\alphabar_{k,\lm1}^-=\Exp\partial \Phat_{k,\lm1}/ \partial P_{k,\lm1}^+$} \label{line:alphan_se}
        \STATE{$\displaystyle P_{\kp1,\lm1}^- =
       \frac{\Phat_{k,\lm1}-P^0_{\lm1}-\alphabar_{k,\lm1}^-P_{k,\lm1}^+}{1-\alphabar^-_{k,\lm1}}$}
        \label{line:pn_se} 
        \STATE{$\gammabar_{\kp1,\lm1}^- = (\tfrac{1}{\alphabar_{k,\lm1}^-}-1)\gammabar_{k,\lm1}^+$}
        \label{line:gamn_se}
        \STATE{$Q^-_{\kp1,\lm1} \sim \Norm(0,\tau_{\kp1,\lm1}^-)$,
        $\tau_{k,\lm1}^- = \Exp(P^-_{\kp1,\lm1})^2$}
        \label{line:qn_se}
    \ENDFOR

\ENDFOR

\end{algorithmic}
\end{algorithm} 
The SE for the general ML-VAMP algorithm are derived in
\cite{pandit2019asymptotics} and the special case of the updates
for ML-VAMP for GLM learning are shown in
Algorithm~\ref{algo:ml_vamp_se} with details of functions $\gbf_\ell^\pm$ in Appendix \ref{app:ml_vamp_details}. 
We see that the SE
updates in Algorithm~\ref{algo:ml_vamp_se} parallel
those in the ML-VAMP algorithm  Algo.~\ref{algo:ml_vamp},
except that vector quantities such as $\zbfhat_{k\ell}$, $\pbfhat_{k\ell}$, $\rbf_{k\ell}^+$ and $\rbf_{k\ell}^-$
are replaced by scalar random variables such as $\Zhat_{k\ell}$, $\Phat_{k\ell}$, $R_{k\ell}^+$ and $R_{k\ell}^-$.
Each of these random variables are described by the 
deterministic parameters such as $\Kbf_{k\ell}\in\Real^{2\times 2}_{\succ 0}$, and $\tau^0_\ell$,
$\tau^-_{k\ell} \in\Real_+$.

The updates in the section labeled as ``Initial'',
provide the scalar equivalent model for the true system \eqref{eq:sig_flow_true}.
In these updates, $\Xi_\ell$ represent the limits of the 
vectors $\xibf_\ell$ in \eqref{eq:xidef}.  That is,
\beq \label{eq:xi_rv_def}
    \Xi_1 := S_{\rm tr}, \quad 
    \Xi_2 := S_{\rm mp}^+, \quad
    \Xi_3 := D.
\eeq
Due to assumptions in Section~\ref{sec:lsl},
we have that the components of $\xibf_\ell$
converge empirically as,
\beq \label{eq:xilim}
    \lim_{N \rightarrow \infty} \{ \xi_{\ell,i} \}
        \stackrel{PL(2)}{=} \Xi_\ell,
\eeq
So, the $\Xi_\ell$ represent the asymptotic
distribution of the components of the vectors $\xibf_\ell$.

The updates in sections labeled ``Forward pass"
and ``Backward pass" in the SE equations in Algorithm~\ref{algo:ml_vamp_se}
parallel those in Algorithm~\ref{algo:ml_vamp}.
The key quantities in these SE equations are the 
error variables,
\[
    \pbf_{k\ell}^+ := \rbf_{k\ell}^+-\pbf^0_\ell, \quad
    \qbf_{k\ell}^- := \rbf_{k\ell}^--\zbf^0_\ell, \quad
\]
which represent the errors of the estimates to the inputs
of the denoisers.  We will also be interested in their
transforms,
\[
    \qbf_{k\ell}^+ = \Vbf_{\ell}\tran \pbf_{k,\lp1}^+,
    \quad
    \pbf_{k\ell}^- = \Vbf_\ell\qbf_{k\ell}^-.
\]
The following Theorem is an adapted version of the main result from 
\cite{pandit2019asymptotics} to the iterates of Algorithms \ref{algo:ml_vamp} and \ref{algo:ml_vamp_se}.

\begin{theorem} \label{thm:se_gen}
Consider the outputs of the ML-VAMP
for GLM Learning Algorithm under the assumptions of 
Section~\ref{sec:lsl}.  Assume the denoisers satisfy
the continuity conditions in Assumption~\ref{as:continuity}.
Also, assume that the outputs of the SE satisfy
\[
    \alphabar_{k\ell}^{\pm} \in (0,1),
\]
for all $k$ and $\ell$.  Suppose Algo. \ref{algo:ml_vamp} is initialized so that the following convergence holds
\begin{align*}
\lim_{N\rightarrow\infty}
    \{\rbf^-_{0\ell}-\zbf_\ell^0\}\overset{PL(2)}=Q_{0\ell}^-
\end{align*}
where $(Q_{00}^-,Q_{01}^-,Q_{02}^-)$ are independent zero-mean Gaussians, independent of $\{\Xi_\ell\}$.
Then, 
\begin{enumerate}[label=(\alph*),topsep=0pt]
    \item  For any fixed iteration
$k \geq 0$ in the forward direction
and layer $\ell=1,\ldots,\Lm1$,
we have that, almost surely,
\begin{align}  \label{eq:alpha_lim_fwd}
    \MoveEqLeft\lim_{N \rightarrow \infty} 
    (\gamma_{k,\lm1}^+,\gamma_{k\ell}^-)
    = (\gammabar_{k,\lm1}^+,\gammabar_{k\ell}^-),\quad{\rm and},\\
    \MoveEqLeft \lim_{N \rightarrow \infty} \{(\zbfhat^+_{k\ell},\zbf^0_\ell,\pbf_{\lm1}^0,\rbf^+_{k,\lm1},
    \rbf^-_\ell) \}  \nonumber \\
    &\overset{PL(2)}= (\Zhat^+_{k\ell},Z^0_\ell,P_{\lm1}^0,R^+_{k,\lm1},
    R^-_\ell) \label{eq:var_lim_fwd}
\end{align}
where the variables on the right-hand side are from the SE
equations \eqref{eq:alpha_lim_fwd} and 
\eqref{eq:var_lim_fwd} are the outputs of the SE equations 
in Algorithm~\ref{algo:ml_vamp_se}.
Similar equations hold for $\ell=0$ with the appropriate
variables removed.

    \item  Similarly, in the reverse direction,
    For any fixed iteration
$k \geq 0$ and layer $\ell=1,\ldots,L-2$,
we have that, almost surely,
\begin{align}  \label{eq:alpha_lim_rev}
    \MoveEqLeft\lim_{N \rightarrow \infty} 
    (\gamma_{k,\lm1}^+,\gamma_{\kp1,\ell}^-)
    = (\gammabar_{k,\lm1}^+,\gammabar_{\kp1,\ell}^-),\quad {\rm and}\\
    \MoveEqLeft \lim_{N \rightarrow \infty} \{(\pbfhat^+_{\kp1,\lm1},\zbf^0_\ell,\pbf_{\lm1}^0,\rbf^+_{k,\lm1},
    \rbf^-_{\kp1,\ell}) \}  \nonumber \\
   \MoveEqLeft\quad\overset{PL(2)}= (\Phat^+_{\kp1,\lm1},Z^0_\ell,P_{\lm1}^0,R^+_{k,\lm1},
    R^-_{\kp1,\ell}). \label{eq:var_lim_rev}
\end{align}
\end{enumerate}
Furthermore, $(P_{\lm1}^0,P_{k\lm1}^+)$ and $Q_{k\ell}^-$ are independent.
\end{theorem}
\begin{proof}
This is a direct application of Theorem 3 from \cite{pandit2019inference} to the iterations of Algorithm \ref{algo:ml_vamp}.  
The convergence result in \cite{pandit2019inference}
requires the uniform Lipschitz continuity 
of functions $\gbf_\ell^{\pm}(\cdot)$.  Assumption~\ref{as:continuity}
provides the required uniform Lipschitz continuity assumption
on $\gbf_0^+(\cdot)$ and $\gbf_3^-(\cdot)$.  For the ''middle" 
layers, $\ell=1,2$, the denoisers $\gbf_\ell^{\pm}(\cdot)$
are linear and the uniform continuity assumption is valid
since we have made the additional assumption that the
terms $\sbf_{\rm tr}$ and $\sbf_{\rm mp}$ are bounded almost surely.
\end{proof}

A key use of the Theorem is to compute
asymptotic empirical limits.  Specifically, 
for a componentwise function $\psi(\cdot)$,
let $\inner{\psi(\xbf)}$ denotes the average $\tfrac1N\sum_{n=1}^N \psi(x_n)$
The above theorem then states that for any componentwise pseudo-Lipschitz function $\psi(\cdot)$ of order 2,
as $N\rightarrow\infty$, we have the following two properties
\begin{align*}
    \MoveEqLeft
    \lim_{N \rightarrow \infty} 
    \inner{\psi(\zbfhat^+_{k\ell},\zbf^0_\ell,\pbf_{\lm1}^0,\rbf^+_{k,\lm1},
    \rbf^-_\ell)} \nonumber\\
    &=\Exp\,\psi(\Zhat^+_{k\ell},Z^0_\ell,P_{\lm1}^0,R^+_{k,\lm1},
    R^-_\ell)\\
    \MoveEqLeft
    \lim_{N \rightarrow \infty} \inner{\psi(\pbfhat^+_{\kp1,\lm1},\zbf^0_\ell,\pbf_{\lm1}^0,\rbf^+_{k,\lm1},
    \rbf^-_{\kp1,\ell)}}\\
    &=\Exp\,\psi(\Phat^+_{\kp1,\lm1},Z^0_\ell,P_{\lm1}^0,R^+_{k,\lm1},
    R^-_{\kp1,\ell}).
\end{align*}
That is, we can compute the empirical average over components
with the expected value of the random variable limit.
This convergence is key to proving Theorem \ref{thm:thm1_complete}.

\section{Empirical Convergence of Fixed Points}
\label{app:conv}

A consequence of Assumption~\ref{as:conv}
is that we can take the limit $k \rightarrow \infty$ 
of the random variables in the SE algorithm.
Specifically, let 
$\xbf_k=\xbf_k(N)$ be any set of $d$ outputs
from the ML-VAMP for
GLM Learning Algorithm under the assumptions
of Theorem~\ref{thm:se_gen}.  Under 
Assumption~\ref{as:conv}, for each $N$, there exists 
a vector 
\beq \label{eq:xlimk}
    \xbf(N) = \lim_{k \rightarrow \infty} \xbf_k(N),
\eeq
representing the limit over $k$.  For each $k$, 
Theorem~\ref{thm:se_gen} shows there also
exists a random vector limit,
\beq \label{eq:Xlimn}
    \lim_{N \rightarrow \infty} \{ \xbf_{k,i}(N) \}
    \stackrel{PL(2)}{=} X_k,
\eeq
representing the limit over $N$.  The following proposition
shows that we can take the limits of the random 
variables $X_k$.

\begin{proposition} \label{prop:conv_var}
Consider the outputs of the ML-VAMP for
GLM Learning Algorithm under the assumptions
of Theorem~\ref{thm:se_gen} and 
Assumption~\ref{as:conv}.
Let $\xbf_k=\xbf_k(N)$ be any set of $d$ outputs
from the algorithm and let $\xbf(N)$ be its limit
from \eqref{eq:xlimk} and let $X_k$ be the random 
variable limit \eqref{eq:Xlimn}.
Then, there exists 
a random variable $X \in \R^d$ such that,
for any pseudo-Lipschitz continuous $f:\R^d \rightarrow \R$,
\beq \label{eq:fconv_lem}
    \lim_{k \rightarrow \infty} \Exp f(X_k) = \Exp f(X) 
    = \lim_{N \rightarrow \infty} \frac{1}{N}
        \sum_{i=1}^N f(x_i(N)).
\eeq
In addition, the SE parameter limits 
$\alphabar_{k\ell}^{\pm}$ and
$\gammabar_{k\ell}^{\pm}$ converge to limits,
\beq \label{eq:param_lim}
    \lim_{k \rightarrow \infty} \alphabar_{k\ell}^{\pm}
    = \alphabar_{\ell}^{\pm}, \quad
    \lim_{k \rightarrow \infty} \gammabar_{k\ell}^{\pm}
    = \gammabar_{\ell}^{\pm}.
\eeq
\end{proposition}

\medskip
The proposition shows that, under the convergence
assumption, Assumption~\ref{as:conv}, we can take
the limits as $k \rightarrow \infty$ of
the random variables from the SE.  To prove the 
proposition we first need the following simple lemma.

\begin{lemma} \label{lem:ab}  
If $\alpha_N$ and $\beta_k \in \R$ are sequences such that
\beq \label{eq:alpha_beta}
    \lim_{k \rightarrow \infty} 
    \lim_{N \rightarrow \infty} |\alpha_N - \beta_k | = 0,
\eeq
then, there exists a constant $C$ such that,
\beq \label{eq:alpha_beta_lim}
    \lim_{N \rightarrow \infty} \alpha_N = 
    \lim_{k \rightarrow \infty} \beta_k
    = C.
\eeq
In particular, the two limits in \eqref{eq:alpha_beta_lim}
exist.
\end{lemma}
\begin{proof} For any $\epsilon > 0$, the limit
\eqref{eq:alpha_beta} implies that there exists a $k_\epsilon(\uparrow\infty$ as $\epsilon\downarrow0)$
such that for all $k > k_\epsilon$,
\[
    \lim_{N \rightarrow \infty} |\alpha_N - \beta_k| < 
    \epsilon,
\]
from which we can conclude,
\[
    \liminf_{N \rightarrow \infty} \alpha_N > \beta_k - \epsilon
\]
for all $k > k_\epsilon$.  Therefore, 
\[
    \liminf_{N \rightarrow \infty} \alpha_N \geq 
    \sup_{k\geq k_\epsilon} \beta_k - \epsilon.
\]
Since this is true for all $\epsilon > 0$, it follows that
\beq \label{eq:limab_inf}
    \liminf_{N \rightarrow \infty} \alpha_N \geq 
    \limsup_{k\rightarrow \infty} \beta_k.
\eeq
Similarly, $\limsup_{N \rightarrow \infty} \alpha_N \leq 
    \inf_{k>k_\epsilon} \beta_k+\epsilon$, whereby
\beq \label{eq:limab_sup}
    \limsup_{N \rightarrow \infty} \alpha_N \leq 
    \liminf_{k\rightarrow \infty} \beta_k.
\eeq
Equations \eqref{eq:limab_inf} and
\eqref{eq:limab_sup} together show
that the limits in \eqref{eq:alpha_beta_lim}
exists and are equal.
\end{proof}

\medskip 
\noindent
\paragraph*{Proof of Proposition~\ref{prop:conv_var}}
Let $f:\R^d \rightarrow \R$ be any pseudo-Lipschitz 
function of order 2, and define,
\beq \label{eq:ab_def}
    \alpha_N = \frac{1}{N}\sum_{i=1}^N f(x_i(N)), 
    \quad
    \beta_k = \Exp f(X_k).
\eeq
Their difference can be written as,
\beq \label{eq:ab_diff}
    \alpha_N - \beta_k = A_{N,k} + B_{N,k},
\eeq
where
\begin{align}
    A_{N,k} &:= \frac{1}{N} \sum_{i=1}^N
         f(\xbf_{i}(N)) - f(\xbf_{k,i}(N)) ,
                \label{eq:Ank} \\
    B_{N,k} &:= \frac{1}{N}\sum_{i=1}^N
        f(\xbf_{k,i}(N)) - \Exp f(X_k).
        \label{eq:Bnk}
\end{align}
Since $\{ x_{k,i}(N) \}$ converges $PL(2)$ to $X_k$,
we have,
\beq \label{eq:Bnk_lim}
    \lim_{N \rightarrow \infty} B_{N,k} = 0.
\eeq
For the term $A_{N,k}$,
\begin{align}
    \MoveEqLeft |A_{N,k}|
   \stackrel{\rm (a)}\leq
    \lim_{N \rightarrow \infty}
     \frac{1}{N}\sum_{i=1}^N
        \left| f(\xbf_{i}(N)) - f(\xbf_{k,i}(N)) \right| \nonumber \\
    &\stackrel{\rm (b)}{\leq} \lim_{N \rightarrow \infty}
     \frac{C}{N} \sum_{i=1}^N
         a_{ki}(N)(1+a_{ki}(N)) \nonumber \\
    &\stackrel{\rm (c)} \leq 
    C \lim_{N \rightarrow \infty}
    \sqrt{ \frac{1}{N} \sum_{i=1}^N a^2_{ki}(N) }
    + 
    \frac{1}{N} \sum_{i=1}^N a_{ki}^2(N) \nonumber \\
    &= C \lim_{N \rightarrow \infty} \epsilon_{k}(N)(1+\epsilon_{k}(N)), \label{eq:Ank_lim1}
\end{align}
where (a) follows from applying the triangle inequality to the definition of $A_{N,k}$ in 
\eqref{eq:Ank}; 
(b) follows from the definition of pseudo-Lipschitz 
continuity in Definition~\ref{def:pl_cont}, 
$C > 0$ is the Lipschitz contant and
\[
    a_{ki}(N) := \|\xbf_{k,i}(N)-\xbf_{i}(N)\|_2,
\]
and (c) follows from the RMS-AM inequality:
\begin{align*}
    \MoveEqLeft 
    \left( \frac{1}{N} \sum_{i=1}^N a_{ki}(N) \right)^2 \leq 
     \frac{1}{N} \sum_{i=1}^N a_{ki}^2(N) =: 
     \epsilon_{k}^2(N).
\end{align*}
By \eqref{eq:xconv}, we have that,
\[
    \lim_{k\rightarrow \infty} \lim_{N \rightarrow \infty} \epsilon_{k}(N) = 0.
\]
Hence, from \eqref{eq:Ank_lim1}, it follows that,
\beq \label{eq:Ank_lim2}
    \lim_{k\rightarrow \infty} \lim_{N \rightarrow \infty} A_{N,k} = 0.
\eeq
Substituting \eqref{eq:Bnk_lim} and \eqref{eq:Ank_lim2}
into \eqref{eq:ab_diff} show that $\alpha_N$ and $\beta_k$
satisfy \eqref{eq:alpha_beta}.
Therefore, applying Lemma \ref{lem:ab} we have
that for any pseudo-Lipschitz function $f(\cdot)$,
there exists a limit $\Phi(f)$ such that,
\beq \label{eq:fconv_lem1}
    \lim_{N \rightarrow \infty} 
    \frac{1}{N}\sum_{i=1}^N f(x_i(N)) 
    = 
    \lim_{k \rightarrow \infty} \Exp f(X_k) = \Phi(f).
\eeq
In particular, the two limits in \eqref{eq:fconv_lem1} 
exists.  When restricted to the continuous,
bounded functions with the $\|f\|_\infty$ norm,
it is easy verified that $\Phi(f)$
is a positive, linear, bounded function of $f$, 
with $\Phi(1)=1$.  Therefore, by the Riesz representation
theorem, there exists 
a random variable $X$ such that $\Phi(f) = \Exp f(X)$.
This fact in combination with \eqref{eq:fconv_lem1}
shows \eqref{eq:fconv_lem}.  

It remains to prove the parameter limits in
\eqref{eq:param_lim}.
We prove the result for the parameter $\wb{\alpha}_{k\ell}^{+}$.
The proof for the other parameters are similar.
Using Stein's lemma, it is shown in 
\cite{pandit2019asymptotics} that
\beq \label{eq:alphabar_frac}
    \wb{\alpha}_{k\ell}^{+} = \frac{ \Exp(\wh{Z}_{k\ell}
    Q^-_{k\ell} )}{\Exp(Q_{\ell}^-)^2}.
\eeq
Since the numerator and denominator of \eqref{eq:alphabar_frac}
are $PL(2)$ functions we have that the limit,
\begin{align}
    \alphabar_{\ell}^+ &:= \lim_{k \rightarrow \infty} \wb{\alpha}_{k\ell}^{+} = 
    \lim_{k \rightarrow \infty}\frac{ \Exp(\wh{Z}_{k\ell}
    Q^-_{k\ell} )}{\Exp(Q_{k\ell}^-)^2} \nonumber \\
    &= 
    \frac{ \Exp(\wh{Z}_{\ell}
    Q^-_{\ell} )}{\Exp(Q_{\ell}^-)^2},
\end{align}
where $\wh{Z}_{\ell}$ and $Q^-_{\ell}$ are the limits
of  $\wh{Z}_{k\ell}$ and $Q^-_{k\ell}$.
This completes the proof.
\hfill\qedsymbol

\section{Proof of Theorem~\ref{thm:thm1_complete}}
\label{app:main_proof}

From Assumption \ref{as:conv}, we know that for every $N$,
every group of vectors $\xbf_k$ converge to limits, 
$\xbf := \lim_{k \rightarrow \infty} \xbf_k$.  The parameters,
$\gamma_{k\ell}^{\pm}$, also converge to limits
$\gammabar_{\ell}^{\pm}:=\lim_{k \rightarrow \infty} \gamma_{k\ell}^{\pm}$
for all $\ell$.  By the continuity assumptions on the 
functions $\gbf_\ell^{\pm}(\cdot)$, the limits 
$\xbf$ and $\gammabar_{\ell}^{\pm}$ are fixed points of the algorithms.

A proof similar to that in \cite{pandit2019inference} shows that
the fixed points $\zbfhat_\ell$ and $\pbfhat_\ell$ satisfy
the KKT condition of the constrained optimization
\eqref{eq:zpmin1}. This proves part (a).

The estimate $\wbfhat$ is the limit,
\[
    \wbfhat = \zbfhat_0 = \lim_{k \rightarrow \infty}
        \zbfhat_{k0}.
\]
Also, the true parameter is $\zbf_0^0=\wbf^0$.
By Proposition~\ref{prop:conv_var}, we have that the $PL(2)$
limits of these variables are
\begin{align*}
    \lim_{N \rightarrow \infty} 
    \{(\wbfhat,\wbf_0)\} \stackrel{PL(2)}{=}
    (\wh W,W_0) := (\wh Z_0,Z_0^0).
\end{align*}
From line \ref{line:zhatp0_se} of the SE Algorithm~\ref{algo:ml_vamp_se}, we have 
\begin{align*}
\wh W=\wh Z_0=g_0^+(R_0^-,\gammabar_0^-)=\prox_{f_{\rm in}/\gammabar_0^-}(W^0+Q_0^-).
\end{align*}
This proves part (b).

To prove part (c), we use the limit
\begin{align}\label{eq:p_fixed_point_convergence}
    \lim_{N \rightarrow \infty} \{p^0_{0,n},\wh p_{0,n}\} \PLeq (P_0^0,\wh P_{0}).
\end{align}
Since the fixed points are critical points
of the constrained optimization \eqref{eq:zpmin1},
$\wh\pbf_0 = \Vbf_0\wbfhat$.  We also have
$\pbf^0_0 = \Vbf_0\wbf^0$.
Therefore,
\begin{align}
    \MoveEqLeft \left[ z_{\rm ts}^{(N)}\  
        \wh{z}_\ts^{(N)} \right] := 
    \ubf\tran\diag(\sbf_\ts)\Vbf_0[\wbf^0\ \wbfhat] 
    \nonumber \\ 
    &=\ubf\tran\diag(\sbf_\ts)[\pbf^0_0\ \pbfhat_0].
\end{align}
Here, $(N)$ in the subscript denotes the dependence on N.  Since $\ubf \sim \Norm(0, \tfrac{1}{p}\Ibf)$,
$[z_\ts^{(N)}\ \wh z_\ts^{(N)}]$ 
is a zero-mean bivariate Gaussian with covariance matrix
\beq \nonumber
    \Mbf^{(N)}= \tfrac{1}{p} \sum_{n=1}^{p}
    \begin{bmatrix}
    s_{\ts,n}^2 p^0_{0,n} p^0_{0,n} & 
    s_{\ts,n}^2 p^0_{0,n} \wh p_{0,n} \\
    s_{\ts,n}^2 p^0_{0,n} \wh p_{0,n} & 
    s_{\ts,n}^2 \wh p_{0,n} \wh p_{0,n}
    \end{bmatrix}
\eeq
The empirical convergence \eqref{eq:p_fixed_point_convergence} yields the following limit,
\begin{align}\label{eq:M_def}
\lim_{N\rightarrow\infty}\Mbf^{(N)} = \Mbf:=
\Exp\,S_{\ts}^2\begin{bmatrix}
     P^0_{0} P^0_{0} & 
     P^0_{0} \wh P_{0} \\
     P^0_{0} \wh P_{0} & 
     \wh P_{0} \wh P_{0}
    \end{bmatrix}.
\end{align}

It suffices to show that the distribution of $[z_\ts^{(N)}\,\wh z_\ts^{(N)}]$ converges to the distribution of $[Z_\ts\,\wh Z_\ts]$ in the Wasserstein-2 metric as $N\rightarrow\infty.$ (See the discussion in Appendix \ref{app:definitions} on the equivalence of convergence in Wasserstein-2 metric and PL(2) convergence.)

Now, 
Wassestein-2 distance between between two probability measures $\nu_1$ and $\nu_2$ is defined as
\begin{equation}
W_2(\nu_1, \nu_2) = \left(\inf_{\gamma\in\Gamma}\Exp_\gamma \norm{X_1 - X_2}^2\right)^{1/2},\label{eq:W2_distance}    
\end{equation}
where $\Gamma$ is the set of probability distributions on the product space with marginals consistent with $\nu_1$ and $\nu_2$. 
For Gaussian measures $\nu_1 =\mcN(\zero,\Sigma_1)$ and $\nu_2=\mcN(\zero,\Sigma_2)$ we have \cite{givens1984class}
\begin{equation*}
    W^2_2(\nu_1, \nu_2) = 
    \tr(\Sigma_1 - 2(\Sigma_1^{1/2}\Sigma_2\Sigma_1^{1/2})^{1/2}+\Sigma_2).
\end{equation*}
Therefore, for Gaussian distributions $\nu_1^{(N)}=\mc N(\zero,\Mbf^{(N)})$, and $\nu_2=\mc N(\zero,\Mbf)$, the convergence \eqref{eq:M_def} implies $W_2(\nu_1^{(N)},\nu_2)\rightarrow0,$ \ie, convergence in Wasserstein-2 distance. Hence,
\begin{equation*}
    (z_{\ts}^{(N)},\wh z_{\ts}^{(N)}) \stackrel{W_2}{\longrightarrow}(Z_\ts, \Zhat_\ts)
    \sim \Norm(0,\Mbf),
\end{equation*}
where $\Mbf$ is the covariance matrix in \eqref{eq:M_def}.
Hence the convergence holds in the PL(2) sense (see discussion in Appendix \ref{app:definitions} on the equivalence of convergence in $W_2$ and PL(2) convergence).

Hence the asymptotic generalization error 
\eqref{eq:Etest} is
\begin{align}
    \MoveEqLeft {\mathcal E}_{\rm ts} := 
    \lim_{N \rightarrow \infty} 
    \Exp f_{\rm ts}(\wh{y}_{\rm ts},y_{\rm ts})
    \nonumber \\
    & \stackrel{(a)}{=}
    \lim_{N \rightarrow \infty} 
    \Exp f_{\rm ts}(\phi_{\rm out}(z_{\ts}^{(N)},D),
    \phi(\wh z_{\ts}^{(N)})) \nonumber \\
    & \stackrel{(b)}{=}
    \Exp f_{\rm ts}(\phi_{\rm out}(Z_{\ts},D),
    \phi(\wh Z_{\ts})),
\end{align}
where (a) follows from \eqref{eq:ytest};
and step (b) follows from continuity assumption 
in Assumption~\ref{as:continuity}(b) along with
the definition of PL(2) convergence in Def. \ref{def:plp_convergence}.
This proves part (c).

\section{Formula for $\Mbf$} \label{sec:M_formula}

For the special cases in the next Appendix,
it is useful to derive expressions for the entries
the covariance matrix $\Mbf$ in \eqref{eq:M_def}.
For the term $m_{11}$, 
\beq\label{eq:m11}
\begin{aligned}
    m_{11} = \Exp\,S_{\ts}^2(P_0^0)^2    =\Exp\,S_{\ts}^2\Exp(P_0^0)^2=\Exp \,S_{\ts}^2 \cdot k_{11},
\end{aligned}
\eeq
where we have used the fact that $P_0^0\independent (S_\ts,S_\tr)$. Next,
$m_{12} = \Exp\,S_{\ts}^2\ P_0^0 \Phat_0.
$
where,
\begin{align}\nonumber
     \Phat_0 
        &= g_1^-(P_0^0+P_0^+,
        Z_1^0+Q_1^-, \gammabar_0^+,\gammabar_1^-, S_{\rm tr}^-)\\
    &= \frac{
    \wb\gamma^+_{0}P_0^+ + S_{\rm tr}\wb\gamma^-_{1}Q_1^-}{\wb\gamma^+_{0} + S^2_{\rm tr}\wb\gamma^-_{1}}+P^0_0,
\end{align}
where $(P_0^0,P_0^+,Q_0^-)$ are independent of $(S_{\tr},S_\ts)$.
Hence,
\begin{align}\nonumber
    \MoveEqLeft m_{12} = \Exp\,S_{\ts}^2\cdot \Exp(P_0^0)^2 +\Exp\frac{S_{\ts}^2 \gammabar_0^+}{\gammabar_0^++S_{\rm tr}^2 \gammabar_1^-}\Exp[P_0^0 P_0^+] \\
    &= m_{11}+ \Exp\left(\frac{S_\ts^2\wb\gamma_0^+}{\wb\gamma_0^++S_\tr^2\wb\gamma_1^-}\right)\cdot k_{12},\label{eq:m12}
\end{align}
since $\Exp[P^0_0Q_1^-]=0$ and $\Kbf_0^+$ is the covariance matrix of $(P_0^0,P_0^+)$ from line \ref{line:Kp_se}. 

Finally for $m_{22}$ we have,
\begin{align}
    \MoveEqLeft m_{22} = \Exp\,S_{\ts}^2\Phat_0 \Phat_0\nonumber\\
    &= \Exp\left(\frac{S_{\ts} \gammabar_0^+}{\gammabar_0^+ +S_{\rm tr}^2 \gammabar_1^-}\right)^2 \Exp(P_0^+)^2  \nonumber\\ &+ \Exp \left(\frac{S_{\ts}S_{\tr} \gammabar_1^-}{\gammabar_0^++S_{\rm tr}^2 \gammabar_1^-}\right)^2 \Exp(Q_1^-)^2  
    \nonumber\\
    &+ \Exp\,S_\ts^2\Exp (P_0^0)^2+ 2\Exp\frac{\gammabar_0^+S_\ts^2}{\gammabar_0^++\gammabar_1^- S_\tr^2}\cdot\Exp\,P_0^0P_0^+
    \nonumber\\
    &= k_{22}\Exp\left(\frac{S_{\ts} \gammabar_0^+}{\gammabar_0^+ +S_{\rm tr}^2 \gammabar_1^-}\right)^2   \nonumber\\ &+ \tau_1^-\Exp \left(\frac{S_{\ts}S_{\tr} \gammabar_1^-}{\gammabar_0^++S_{\rm tr}^2 \gammabar_1^-}\right)^2   - m_{11} + 2m_{12}.\label{eq:m22}
\end{align}

\section{Special Cases}
\label{app:special_cases}

\subsection{Linear Output with Square Error}
In this section we examine a few special cases of the GLM problem \eqref{eq:whatmin}. 
We first consider a linear output with
additive Gaussian noise and 
a squared error training and test loss.
Specifically, consider the model, 
\begin{equation}
    \ybf = \Xbf \wbf^0 + \dbf
\end{equation}
We consider estimates of $\wbf^0$ such that:
\begin{equation}\label{eq:LR}
   \wh \wbf = \argmin_{\wbf} \ \tfrac{1}{2}\norm{\ybf - \Xbf \wbf}^2 + \tfrac{\lambda}{2\black{\beta}} \norm{\wbf}^2
\end{equation}
\black{The factor $\beta$ is added above since the two terms scale with a ratio of $\beta$. It does not change analysis.}
Consider the ML-VAMP GLM learning algorithm applied to this problem. The following corollary follows from the Main result in Theorem \ref{thm:thm1_complete}.

\begin{corollary}[Squared error]\label{cor:squared_error}
For linear regression, \ie, $\phi(t)=t,$ $\phi_{\rm out}(t,d)=t+d,$ $f_\ts(y,\wh y)=(y_\ts-\wh y_\ts)^2$,
$F_{\rm out}(\p_2)=\tfrac1N\norm{\y-\p_2}^2$, we have
\begin{align*}
\mc E_\ts^{\mathsf{LR}}\! =\! \Exp\left(\tfrac{\wb\gamma_0^+ S_\ts}{\wb\gamma_0^++S_\tr^2\wb\gamma_1^-}\right)^2 k_{22}
+   \Exp \left(\tfrac{\wb\gamma_1^- S_\tr S_\ts}{\wb\gamma_0^++S_\tr^2\wb\gamma_1^-}\right)^2\tau_1^-+\sigma_{d}^2.\end{align*}
The quantities $k_{22}$, $\tau_1^-,\gammabar_0^+,\gammabar_1^-$ depend on the choice of regularizer $\lambda$ and the covariance between features.
\end{corollary}
\begin{proof}
This follows directly from the following observation:
\begin{align*}
    \mc E^{\mathsf{SLR}}_\ts &= \Exp(Z_\ts+D-\wh Z_\ts)^2
    = \Exp(Z_\ts-\wh Z_\ts)^2+\Exp\,D^2\\
    &=m_{11}+m_{22}-2m_{12}+\sigma_d^2.
\end{align*}
Substituting equation \eqref{eq:m22} proves the claim.
\end{proof}

\subsection{Ridge Regression with i.i.d.\ 
Covariates} \label{sec:lin_reg}

We next the special case when the input features are independent, \ie, \eqref{eq:LR} where rows of $\Xbf$ corresponding to the training data has i.i.d Gaussian features with covariance $\Pbf_{\rm train} = \frac{\sigma_{\rm tr}^2}{p} \Ibf$ and $S_{\rm tr} = \sigma_{\rm tr}$.

Although the solution to \eqref{eq:LR} exists in closed form $(\Xbf\tran\Xbf+\lambda \Ibf)\inv\Xbf\tran\ybf$, we can study the effect of the regularization parameter $\lambda$ on the generalization error $\mc E_\ts$ as detailed in the result below.

\begin{corollary}\label{cor:GE_for_LR}
Consider the ridge regression problem \eqref{eq:LR} with regularization parameter $\lambda>0$. For the squared loss \ie, $f_\ts(y,\wh y)=(y-\wh y)^2$, i.i.d Gaussian features without train-test mismatch, \ie, $S_\tr=S_\ts=\sigma_\tr$, the generalization error $\mc E_{\ts}^{\mathsf{RR}}$ is given by Corollary \ref{cor:squared_error}, with constants
\begin{align*}
    \MoveEqLeft k_{22}=\Var(W^0),\qquad \gammabar_0^+ = \lambda/\beta,\\
    \MoveEqLeft\gammabar_1^- =\begin{cases} \tfrac{1}{G}-\tfrac{\lambda}{\sigma_{\tr}^2}&\beta<1\\
    \frac{\tfrac{\lambda}{\sigma_\tr^2\black{\beta}}(\tfrac1G-\tfrac\lambda{\sigma_\tr^2\black{\beta}})}{\tfrac{\beta-1}G+\tfrac{\lambda}{\sigma_\tr^2\black{\beta}}} & \beta > 1
    \end{cases}
\end{align*}
where $G = G_\mp(-\frac{\lambda}{\sigma_{\rm tr}^2\beta})$, with $G_\mp$ given in Appendix \ref{app:mp_dist}, and $\tau_1^-=\Exp(P_1^-)^2$ where $P_1^-$ is given in equation \eqref{eq:P1-expression} in the proof.
\end{corollary}

\begin{proof}[Proof of Corollary \ref{cor:GE_for_LR}]
We are interested in identifying the following constants appearing in Corollary \ref{cor:squared_error}:
\begin{align*}
    \Kbf_0^+,\tau_1^-,\gammabar_0^+,\gammabar_1^-.
\end{align*}
These quantities are obtained as fixed points of the State Evolution Equations in Algo. \ref{algo:ml_vamp_se}. We explain below how to obtain expressions for these constants. Since these are fixed points we ignore the subscript $k$ corresponding to the iteration number in Algo. \ref{algo:ml_vamp_se}.

In the case of problem \eqref{eq:LR}, the maps $\prox_{f_{\rm in}}$ and $\prox_{f_{\rm out}}$, \ie, $g_0^+$ and $g_3^-$ respectively, can be expressed as closed-form formulae. This leads to simplification of the SE equations as explained below.  

We start by looking at the \textit{forward pass} (finding quantities with superscript '+') of Algorithm \ref{algo:ml_vamp_se} for different layers and then the \textit{backward pass} (finding quantities with superscript '-') to get the parameters $\{\Kbf_\ell^+,\tau_\ell^-,\wb\alpha_\ell^\pm,\wb\gamma_\ell^\pm\}$ for $\ell=0,1,2$.

To begin with, notice that $f_{\rm in}(w) = \frac{\lambda}{2}w^2$, and therefore the denoiser $g_0^+(\cdot)$
in \eqref{eq:G0def_app} is simply,
\begin{align*}
    g_0^+(r_0^-,\gamma_0^-) %
    &= \tfrac{\gamma_0^-}{\gamma_0^- + \lambda/\black{\beta}}r_0^-,\quad {\rm and}\quad \tfrac{\partial g_0^+}{\partial r_0^-} = \tfrac{\gamma_0^-}{\gamma_0^- + \lambda/\black{\beta}}
\end{align*}
Using the random variable $R_0^-$ and substituting in the expression of the denoiser to get $\wh Z_0$, we can now calculate $\alphabar_0^+$ using lines \ref{line:alphap_se} and \ref{line:gamp_se},
\begin{align}
\alphabar_0^+ = \tfrac{\gammabar_0^-}{\gammabar_0^- + \lambda/\black{\beta}},\qquad \gammabar_0^+ = \lambda/\black{\beta}.
\end{align}

Similarly, we have $f_{\rm out}(p_2) = \half(p_2-y)^2$, whereby
the output denoiser $g_3^-(\cdot)$ in the last layer for ridge regression is given by,
\begin{align}
    g_3^-(r_2^+, \gamma_2^+, y)%
    &= \frac{\gamma_2^+ r_2^+ +y}{\gamma_2^++1}. 
\end{align}
By substituting this denoiser in line \ref{line:phatn0_se} of the algorithm we get $\wh P_2^-$ and thus, following the lines \ref{line:alphan_se}-\ref{line:qn_se} of the algorithm we have
\begin{align}\label{eq:gammabar2n}
    \alphabar_2^- = \black{\tfrac{\gammabar_2^+}{\gammabar_2^++1}}, %
    \qquad{\rm whereby}\quad
    \gammabar_2^- = 1.
\end{align}

Having identified these constants $\alphabar_0^+,\gammabar_0^+,\alphabar_2^-,\gammabar_2^-$, we will now sequentially identify the quantities
\begin{align*}
(\alphabar_0^+,\gammabar_0^+)\rightarrow\Kbf_0^+\rightarrow(\alphabar_1^+,\gammabar_1^+)\rightarrow\Kbf_1^+\rightarrow(\alphabar_2^+,\gammabar_2^+)\rightarrow\Kbf_2^+
\end{align*}
in the forward pass, and then the quantities
\begin{align*}
\tau_0^-\leftarrow(\alphabar_0^-,\gammabar_0^-)\leftarrow\tau_1^-\leftarrow(\alphabar_1^-,\gammabar_1^-)\leftarrow\tau_2^-\leftarrow(\alphabar_2^-,\gammabar_2^-)
\end{align*}
in the backward pass.

We also note that we have
\begin{align}
    \label{eq:alphabar_sum_1}
    \alphabar_\ell^++\alphabar_\ell^-=1
\end{align}

\paragraph{Forward Pass:}
Observe that $\Kbf_0^+ = \Cov(Z_0,Q_0^+)$. Now,
from line \ref{line:qp_se}, on simplification we get $Q_0^+ = -W_0^0$ whereby,
\begin{align}
\Kbf_0^+ = \mathrm{var}(W^0) \begin{bmatrix}
        1 & -1 \\ -1 & 1
    \end{bmatrix}
.
\end{align}
Notice that from line \ref{line:Kp_se}, the pair $(P_0^0,P_0^+)$ is jointly Gaussian with covariance matrix $\Kbf_0^+$. But the above equation means that $P_0^+ = -P^0_0$, whereby $R_0^+=0$ from line \ref{line:rp_se}. 

Now, the linear denoiser $g_1^+(\cdot)$ is defined as in \eqref{eq:G23def}. Note that since we are considering i.i.d Gaussian features for this problem, the random variable $S_{\tr}$ in this layer is a constant $\sigma_{\tr}$.
Therefore, similar to layer $\ell = 0$ by evaluating lines \ref{line:rp_se}-\ref{line:Kp_se} of the algorithm we get $Q_1^+ = -Z_1^0,$ whereby
\begin{align}
    \alphabar_1^+ = \tfrac{\sigma_{\tr}^2 \gammabar_1^-}{\wb\gamma_0^++\sigma_\tr^2\gammabar_1^-},
    \ \gammabar_1^+ = \tfrac{\wb\gamma_0^+}{\sigma_\tr^2}= \tfrac{\lambda}{\sigma_{\tr}^2\black{\beta}},
    \  \Kbf_1^+ &= \sigma_{\tr}^2 \Kbf_0^+.
\end{align}
Observe that this means 
\begin{align}\label{eq:P1p_P10}
P_1^+=-P_1^0.
\end{align}

\paragraph{Backward Pass:}
Since $Y=\phi_{\rm out}(P_2^0,D)=P_2^0+D$, line \ref{line:pn_se} of algorithm on simplification yields $P_2^-=D$, whereby we can get $\tau_2^-$,
\begin{align}
\tau_2^- = \Exp(P_2^-)^2 =\Exp[D^2] = \sigma_d^2.
\end{align}

Next, to calculate the terms $(\alphabar_1^-,\gammabar^-_1)$, we use the decoiser $g_2^-$ defined in \eqref{eq:G23def} for line \ref{line:phatn_se} of the algorithm to get $\wh P_1$. 
\begin{align}\label{eq:P1hat}
\wh P_1  = \tfrac{\gammabar_1^+R_1^++S_\mp^-\gammabar_2^-R_2^-}{\gammabar_1^++(S_\mp^-)^2\gammabar_2^-}=\tfrac{S_\mp^-(S_\mp^+P_1^0+Q_2^-)}{\gammabar_1^++(S_\mp^-)^2},
\end{align}
where we have used $\gammabar_2^-=1,$ $R_1^+=P_1^0+P_1^+=0$ due to \eqref{eq:P1p_P10}, and $R_2^-=Z_2^0+Q_2^-=S_\mp^+P_1^0+Q_2^-$ from lines \ref{line:rp_se},  \ref{line:rn_se} and \ref{line:z0init} respectively.

Then, we calculate $\alphabar_1^-$ and $\gammabar_1^-$ as $\alphabar_1^- = \Exp \frac{\partial g_2^-}{\partial P_1^+}=\Exp\tfrac{\gammabar_1^+}{\gammabar_1^++(S_\mp^-)^2}.$
This gives, 
\begin{align}
    \alphabar_1^- = 
    \begin{cases}
    \tfrac{\lambda}{\sigma_\tr^2\black{\beta}}G & \beta<1\\
    (1-\tfrac1\beta) + \tfrac1\beta\tfrac{\lambda}{\sigma_\tr^2\black{\beta}}G & \beta\geq 1
    \end{cases}
\end{align}
Here, in the overparameterized case $(\beta>1)$, the denoiser $g_2^-$ outputs $R_1^+$ with probability $1-\tfrac1\beta$ and $\tfrac{\lambda}{\sigma_\tr^2\beta}G$ with probability $\tfrac1\beta$.

Next, from line \ref{line:gamn_se} we get,
\begin{align}
    \gammabar_1^- = (\tfrac1{\alphabar_1^-}-1)\gammabar_1^+=
    \begin{cases}
 \tfrac1G-\tfrac\lambda{\sigma_\tr^2\black{\beta}} & \beta<1\\
    \frac{\tfrac{\lambda}{\sigma_\tr^2\black{\beta}}(\tfrac1G-\tfrac\lambda{\sigma_\tr^2\black{\beta}})}{\tfrac{\beta-1}G+\tfrac{\lambda}{\sigma_\tr^2\black{\beta}}} & \beta> 1
    \end{cases}
\end{align}

Now from line \ref{line:pn_se} and equation \eqref{eq:alphabar_sum_1} we get,
\begin{align}
    \alphabar_1^+ P_1^- &= \wh P_1-P_1^0-\alphabar_1^- P_1^+\overset{\rm (a)}=\wh P_1 - \alphabar_1^+ P_1^0\nonumber\\
    &\overset{\rm (b)}=\underbrace{\left(\tfrac{S_\mp^-S_\mp^+}{\tfrac{\lambda}{\sigma_\tr^2\beta}+(S_\mp^-)^2}-\alphabar_1^+\right)}_{A}P_1^0+\underbrace{\tfrac{S_\mp^-}{\tfrac{\lambda}{\sigma_\tr^2\beta}+(S_\mp^-)^2}}_{B}Q_2^-\label{eq:P1-expression}
\end{align}
where (a) follows from \eqref{eq:P1p_P10} and \eqref{eq:alphabar_sum_1}, and (b) follows from \eqref{eq:P1hat}.
From this one can obtain  $\tau_1^-=\Exp(P_1^-)^2$ which can be calculated using the knowledge that $P_1^0,Q_2^-$ are independent Gaussian with covariances $\Exp(P_1^0)^2=\sigma_\tr^2\Var(W^0)$, $\Exp(Q_2^-)^2=\sigma_d^2$. Further, $P_1^0,Q_2^-$ are independent of $(S_\mp^+,S_\mp^-)$. 

Observe that by \eqref{eq:P1-expression} we have 
\begin{align}
    \tau_1^- = \frac{1}{(\alphabar_1^+)^2} \Bigg(\Exp{(A^2)} \sigma_\tr^2\Var(W^0)+ \Exp{(B^2)} \sigma_d^2 \Bigg).
\end{align}
with some simplification we get
\begin{subequations}
\begin{align}
    \Exp{(A^2)} &= (\frac{\lambda}{\sigma_{\rm{tr}}^2 \beta})^2 G' -(\frac{\lambda}{\sigma_{\rm{tr}}^2 \beta} G)^2,\\
    \Exp{(B^2)} &= G - \frac{\lambda}{\sigma_{\rm{tr}}^2 \beta} G',
\end{align}
\end{subequations}
where $G = G_\mp(-\frac{\lambda}{\sigma_{\rm tr}^2\beta})$, with $G_\mp$ given in Appendix \ref{app:mp_dist}, and $G'$ is the derivative of $G_\mp$ calculated at $-\frac{\lambda}{\sigma_{\rm tr}^2\beta}$.

Now consider the \textbf{under-parametrized} case ($\beta<1$):

Let $u = -\frac{\lambda}{\sigma_{\rm{tr}}^2 \beta}$ and $z = G_\mp(u)$. In this case we have 
\begin{align}
    \alphabar_1^+ = 1-\frac{\lambda}{\sigma_{\rm{tr}}^2 \beta} G = 1+uz. 
\end{align}
Note that, 
\begin{subequations}
\begin{align}
    G_\mp^{-1}(z) = u
    &\quad \overset{\rm (a)}\Rightarrow R_\mp(z)+\frac{1}{z} = u \nonumber \\
    &\quad \overset{\rm (b)}\Rightarrow \frac{1}{1-\beta z} + \frac{1}{z} = u, \label{eq:z_fix}
\end{align}
\end{subequations}
where $R_\mp(.)$ is the R-transform defined in \cite{tulino2004random} and (a) follows from the relationship between the R- and Stieltjes-transform and (b) follows from the fact that for Marchenko-Pastur distribution we have $R_\mp(z) = \frac{1}{1-z\beta}$. Therefore,
\begin{align}
    \MoveEqLeft G_\mp(\frac{1}{1-\beta z} + \frac{1}{z}) = z \nonumber\\ \MoveEqLeft \Rightarrow G_\mp'(\frac{1}{1-\beta z} + \frac{1}{z}) = G'= \frac{1}{\frac{\beta}{(1-\beta z)^2}-\frac{1}{z^2}} .
\end{align}
For the \textbf{over-parametrized} case ($\beta >1 $) we have: 
\begin{align}
    \alphabar_1^+ = \tfrac1\beta( 1+\tfrac{\lambda}{\sigma_\tr^2\black{\beta}}G) = \frac{1-uz}{\beta}.
\end{align}
In this case, as mentioned in Appendix \ref{app:mp_dist} and following the results from \cite{tulino2004random}, the measure $\mu_\beta$ scales with $\beta$ and thus $R_\mp(z) = \frac{\beta}{1-z}$. Therefore, similar to \eqref{eq:z_fix}, $z$ satisfies
\begin{align}
     \frac{\beta}{1-z} + \frac{1}{z} = u \quad \Rightarrow G'= \frac{1}{\frac{\beta}{(1-z)^2}-\frac{1}{z^2}}.
\end{align}

Now $\tau_1^-$ can be calculated as follows:
\begin{align}
    \tau_1^- 
    &= \eta^2 \bigg( u^2 z^2 \sigma_{\rm tr}^2 var(W^0) (\kappa-1)+ \sigma_{d}^2 z (uz\kappa +1)  \bigg)
\end{align}
where 
\begin{align}
\eta = 
\begin{cases}
    \frac{1}{(1+uz)} & \beta<1\\
    \frac{\beta}{(1-uz)} &\beta \geq 1
\end{cases}, \quad 
\kappa = 
\begin{cases}
    \frac{ (1-\beta z)^2 }{\beta z^2 - (1-\beta z)^2}  & \beta<1\\
    \frac{ (1-z)^2}{\beta z^2 - (1-z)^2} &\beta \geq 1
\end{cases}
\end{align}

and $z$ is the solution to the fixed points 
\begin{align}
\begin{cases}
    \frac{1}{1-\beta z} + \frac{1}{z} = u  & \beta<1\\
    \frac{\beta}{1-z} + \frac{1}{z} = u &\beta \geq 1
\end{cases}.
\end{align}
\end{proof}
\subsection{Ridgeless Linear Regression}\label{sec:Rless_LR}

Here we consider the case of Ridge regression \eqref{eq:LR} when $\lambda\rightarrow0^+$. Note that the solution to the problem \eqref{eq:LR} is $(\Xbf\tran\Xbf+\lambda \Ibf)\inv\Xbf\tran\ybf$ remains unique since $\lambda>0$. The following result was stated in \cite{hastie2019surprises}, and can be recovered using our methodology. Note however, that we calculate the generalization error whereas they have calculated the squared error, whereby we obtain an additional additive factor of $\sigma_d^2.$ The result explains the double-descent phenomenon for Ridgeless linear regression.

\begin{corollary}
\label{cor:GE_for_LR_lam0}
For ridgeless linear regression, we have
\begin{align*}
    \lim_{\lambda\rightarrow 0^+}\mc E^{\mathsf{RR}}_\ts = \begin{cases}
    \tfrac1{1-\beta}\sigma_d^2 & \beta<1\\
    \frac{\beta}{\beta-1}\sigma_d^2+(1-\tfrac1\beta)\sigma_\tr^2\Var(W^0) &\beta \geq 1
    \end{cases}
\end{align*}
\end{corollary}

\begin{proof}[Proof of Corollary \ref{cor:GE_for_LR_lam0}]
We calculate the parameters $\gammabar_0^+, \gammabar_1^-$, $k_{22}$ and $\tau_1^-$ when $\lambda \rightarrow 0^+$. Before starting off, we note that 
\begin{align}
    G_0 &:= \lim_{z \rightarrow 0^+} G_\mp(-z)= 
     \begin{cases} \frac{\beta}{1-\beta}& \beta<1 \\ \frac{\beta}{\beta-1} &\beta > 1 \end{cases},
\end{align}
as described in Appendix \ref{app:mp_dist}.
Following the derivations in Corollary \ref{cor:GE_for_LR}, we have 
\begin{align}
    \gammabar_0^+ = \lambda/\black{\beta}, \quad k_{22}=\Var(W^0)
\end{align}
Now for $\lambda\rightarrow 0^+,$ we have
\begin{align}
1-\alphabar_1^- = \begin{cases}1&\beta<1\\
    \tfrac1\beta&\beta\geq 1\end{cases},
    \quad
    \gammabar_1^- = \begin{cases}\frac{1}{G_0}=\tfrac{1-\beta}{\beta}&\beta<1\\
    \frac\lambda{(\beta-1)\sigma_\tr^2\black{\beta}}&\beta> 1\end{cases},
\end{align}
Using this in simplifying \eqref{eq:P1-expression} for $\lambda\rightarrow0^+$, we get
\begin{align*}
    \tau_1^- = \Exp(P_1^-)^2=
    \begin{cases}
    \sigma_{d}^2 G_0 &\beta<1\\
    \beta\sigma_{d}^2 G_0+{\sigma_\tr^2\Var(W^0)}{(\beta-1)} & \beta\geq 1
    \end{cases}
\end{align*}
where during the evaluation of $\Exp\left(\frac{S_\mp^-}{\gammabar_1^++(S_\mp^-)^2}\right)^2$, for the case of $\beta>1,$ we need to account for the point mass at $0$ for $S_\mp^-$ with weight $1-\frac1\beta$.

Next, notice that
\begin{align*}
    a:=\frac{\gammabar_0^+\sigma_\tr}{\gammabar_0^++\gammabar_1^-\sigma_\tr^2} = 
    \begin{cases}
    0 &\beta<1\\
    (1-\tfrac1\beta)\sigma_\tr &\beta\geq 1
    \end{cases},
\end{align*} and, 
\begin{align*}
    b:=\frac{\gammabar_1^-\sigma^2_\tr}{\gammabar_0^++\gammabar_1^-\sigma_\tr^2} = 
    \begin{cases}
    1 &\beta<1\\
    \tfrac1\beta &\beta\geq 1
    \end{cases},\quad 
\end{align*}
Thus applying Corollary \ref{cor:squared_error}, we get
\begin{align*}
    \mc E^{\mathsf{RR}}_\ts &= a^2 k_{22}+b^2\tau_1^-+\sigma_d^2\\
    &=
    \begin{cases}
    \tfrac1{1-\beta}\sigma_d^2 & \beta<1\\
    \frac{\beta}{\beta-1}\sigma_d^2+(1-\tfrac1\beta)\sigma_\tr^2\Var(W^0) &\beta \geq 1
    \end{cases}
\end{align*}
This proves the claim.
\end{proof} %
\subsection{Train-Test Mismatch} \label{sec:tr_ts_mismatch}

Observe that our formulation allows for analyzing the effect of mismatch in the training and test distribution. One can consider arbitrary joint distributions over $(S_\tr,S_\ts)$ that model the mismatch between training and test features. Here we give a simple example which highlights the effect of this mismatch.

\begin{definition}[Bernoulli $\varepsilon$-mismatch]\label{def:mismatch} $(S_\ts,S_\tr)$ has a bivariate Bernoulli distribution with
\begin{enumerate}[label=$\bullet$,topsep=0pt, itemsep=5pt]
    \item $\Pr\{S_\tr\!=\!S_\ts\!=\!0\}=\Prob\{S_\tr\!=\!S_\ts\!=\!1\}=(1-\varepsilon)/$2\item $\Pr\{S_\tr\!=\!0,S_\ts\!=\!1\}=\Prob\{S_\tr\!=\!1,S_\ts\!=\!0\}=\varepsilon/2$
\end{enumerate}
\end{definition}
Notice that the marginal distribution of the $S_\tr$ in the Bernoulli $\varepsilon-$mismatch model is such that $\Prob(S_\tr\neq 0)=\frac12.$ Hence half of the features extracted by the matrix $V_0$ are relevant during training. Similarly, $\Prob(S_\ts\neq 0)=\frac12.$ However the features spanned by the test data do not exactly overlap with the features captured in the training data, and the fraction of features common to both the training and test data is $1-\varepsilon$. Hence for $\varepsilon=0$, there is no training-test mismatch, whereas for $\varepsilon=1$ there is a complete mismatch.

The following result shows that the generalization error increases linearly with the mismatch parameter $\varepsilon.$ 
\begin{corollary}[Mismatch]\label{cor:mismatch} Consider the problem of Linear Regression \eqref{eq:LR} under the conditions of Corollary \ref{cor:squared_error}. Additionally suppose we have Bernoulli $\varepsilon$-mismatch between the training and test distributions. Then
\begin{align*}
    \mc E_\ts = \tfrac{k_{22}}2((1-\varepsilon)\gamma^{*2}+\varepsilon)+\tfrac{\tau_{1}^-}2(1-\gamma^*)(1-\varepsilon)+\sigma_{d}^2,
\end{align*}
where $\gamma^*:=\frac{\wb\gamma_0^+}{\wb\gamma_0^++\wb\gamma_1^-}$. The terms $k_{22}, \tau_1^-,\gamma^*$ are independent of $\varepsilon$.
\end{corollary}
\begin{proof}
This follows directly by calculating the expectations of the terms in Corollary \ref{cor:squared_error}, with the joint distribution of $(S_\tr,S_\ts)$ given in Definition \ref{def:mismatch}.
\end{proof}
The quantities $k_{22}$ and $\tau_1^-$ in the result above can be calculated similar to the derivation in the proof of Corollary \ref{cor:GE_for_LR} and can in general depend on the regularization parameter $\lambda$ and overparameterization parameter $\beta$.

 \subsection{Logistic Regression}
The precise analysis for the special case of regularized logistic regression estimator with i.i.d Gaussian features is provided in \cite{salehi2019impact}.
Consider the logistic regression model,
\begin{align*}
    \Prob(y_i = 1 | \xbf_i) := \rho(\xbf_i\tran \wbf) ~~~\text{for } i= 1, \cdots, N
\end{align*}
where $\rho(x) = \frac{1}{1+e^{-x}}$ is the standard logistic function.

In this problem we consider estimates of $\wbf^0$ such that
\begin{align*}
    \wbfhat:= \argmin_{\wbf} \mathbf{1} \tran \log(1+e^{\Xbf \wbf})-\ybf\tran \Xbf \wbf + F_{\rm in}(\wbf).
\end{align*}
where $F_{\rm in}$ is the reguralization function. This is a special case of optimization problem \eqref{eq:whatmin} where
\beq
F_{\rm out}(\ybf, \Xbf\wbf) = \mathbf{1} \tran \log(1+e^{\Xbf \wbf})-\ybf\tran \Xbf \wbf.
\eeq

Similar to the linear regression model, using the ML-VAMP GLM learning algorithm, we can characterize the generalization error for this model with quantities $\Kbf_0^+,\tau_1^-,\gammabar_0^+,\gammabar_1^-$ given by algorithm \ref{algo:ml_vamp_se}. We note that in this case, the output non-linearity is
\begin{align}\label{eq:fout_log}
    \phi_{\rm out}(p_2,d) = \indic{\rho(p_2)>d}
\end{align}
where $d\sim \text{Unif}(0,1)$.
Also, the denoisers $g_0^+$, and $g_3^-$ can be derived as the proximal operators of $F_{\rm in}$, and  $F_{\rm out}$ defined in \eqref{eq:Gb03_def}.

 \subsection{Support Vector Machines}
The asymptotic generalization error for support vector machine (SVM) is provided in \cite{deng2019model}.
Our model can also handle SVMs. Similar to logistic regression, SVM finds a linear classifier using the hinge loss instead of logistic loss. Assuming the class labels are $y=\pm 1$ the hinge loss is
\begin{equation}
    \ell_{\rm hinge}(y,\hat{y}) = \max(0, 1 - y\hat{y}).
\end{equation}
Therefore, if we take
\beq
F_{\rm out}(\ybf, \Xbf\wbf) = \sum_{i} \max(0, 1 - y_i\Xbf_i\wbf),
\eeq
where $\Xbf_i$ is the $i^{th}$ row of the data matrix, the ML-VAMP algorithm for GLMs finds the SVM classifier. The algorithm would have proximal map of hinge loss and our theory provides exact predictions for the estimation and prediction error of SVM. 

As with all other models considered in this work, the true underlying data generating model could be anything that can be represented by the graphical model in Figure \ref{fig:graph_mod}, e.g.\ logistic or probit model, and our theory is able to exactly predict the error when SVM is applied to learn such linear classifiers in the large system limit. %

\section{Marchenko-Pastur distribution}
\label{app:mp_dist}
We describe the random variable $S_{\mp}$ defined in \eqref{eq:smp_lim} where  $S_\mp^2$ has a rescaled Marchenko-Pastur distribution. Notice that the positive entries of $\sbf_\mp$ are the positive eigenvalues of $\Ubf\tran\Ubf$ (or $\Ubf\Ubf\tran$).

Observe that $U_{ij}\sim N(0,\frac1p)$, whereas, the standard scaling while studying the Marchenko-Pastur distribution is for matrices $\Hbf$ such that $H_{ij}\sim \mc N(0,\frac1N)$ (for e.g. see equation (1.10) from \cite{tulino2004random} and the discussion preceding it). Also notice that $\sqrt\beta \Ubf$ has the same distribution as $\Hbf$. Thus the results from \cite{tulino2004random} apply directly to the distributions of eigenvalues of $\beta\Ubf\tran\Ubf$ and $\beta\Ubf\Ubf\tran$. We state their result below taking into account this disparity in scaling. 

The positive eigenvalues of $\beta\Ubf\tran \Ubf$ have an empirical distribution which converges to the following density:
\begin{align}\label{eq:MP_density}
    \mu_\beta(x)=\frac{\sqrt{(b_\beta-x)_+(x-a_\beta)_+}}{2\pi\beta x} 
\end{align}
where $a_\beta=(1-\sqrt{\beta})^2$, $b_\beta:=(1+\sqrt{\beta})^2$. Similarly the positive eigenvalues of $\beta\Ubf\Ubf\tran$ have an empirical distribution converging to the density $\beta\mu_\beta$.
We note the following integral which is useful in our analysis:
\begin{align}\label{eq:G0_def}
G_0:&=\lim_{z\rightarrow 0^-}\Exp\frac{1}{S_{\mp}^2-z}\indic{S_\mp>0}\nonumber \\
&=\lim_{z\rightarrow 0^-}\int_{a_\beta}^{b_\beta} \frac{1}{x/\beta-z}\mu_\beta(x)dx=\frac{\beta}{|\beta-1|}.
\end{align}

More generally, the Stieltjes transform of the density is given by:
\begin{align}\label{eq:Stieltjes}
    G_\mp(z)=\Exp\frac1{S_\mp^2-z}\indic{S_\mp>0}=\int_{a_\beta}^{b_\beta} \frac{1}{ x/\beta-z}\mu_\beta(x)dx
\end{align}

\end{document}